\documentclass[journal]{IEEEtran}

\usepackage{cite}
\usepackage{amsmath,amssymb,amsfonts}
\usepackage{graphicx}
\usepackage{textcomp}
\usepackage{subfig}
\usepackage[flushleft]{threeparttable} % For the comparison table
\usepackage{booktabs}
\usepackage{algorithm} 
\usepackage{algpseudocode}
\usepackage[colorlinks = true,
            linkcolor = blue,
            urlcolor  = blue,
            citecolor = blue,
            anchorcolor = blue]{hyperref}
\usepackage{microtype}
\allowdisplaybreaks

\usepackage{url}

\usepackage{comment}

\usepackage{multicol}
\usepackage{multirow}
\usepackage{dsfont}
\usepackage{mathtools}
\usepackage{cuted}
\usepackage[makeroom]{cancel}
\usepackage{stmaryrd}
\usepackage{booktabs, makecell}
\usepackage{pifont}% http://ctan.org/pkg/pifont
\usepackage{color}

\definecolor{b2}{RGB}{51,153,255}
\definecolor{p2}{RGB}{121,64,255}

\newcommand{\ie}{{\it{i.e.,~}}}
\newcommand{\eg}{{\it{e.g.,~}}}

\newcommand{\etal}{{\it{ et al.~}}}

\newcommand{\R}{\mathcal{R}}
\newcommand{\T}{\mathcal{T}}
\newcommand{\bE}{\mathbb{E}}
\newcommand{\cL}{\mathcal{L}}
\newcommand{\bW}{\mathbf{W}}
\newcommand{\bw}{\mathbf{w}}

\newcommand{\bR}{\mathbb{R}}
\newcommand{\bY}{\mathbf{Y}}

\newcommand{\ours}{\textbf{\texttt{GATE}}}

\DeclareMathOperator*{\argmin}{arg\,min}

\newcommand{\revisionR}[1]{{#1}}
\newcommand{\revision}[1]{#1}
\newcommand{\revisionfoot}[1]{}

\usepackage{setspace}
\usepackage{arydshln} 
\usepackage[normalem]{ulem}
\newtheorem{theorem}{Theorem}
\newtheorem{proof}[theorem]{Proof}

\newtheorem{remark}[theorem]{Remark}
\newtheorem{lemma}[theorem]{Lemma}

\usepackage{url}
\usepackage{rotating}
\usepackage{comment}
\usepackage{multicol}
\usepackage{multirow}
\usepackage{dsfont}
\usepackage{mathtools}
\usepackage{cuted}
\usepackage[makeroom]{cancel}
\usepackage{stmaryrd}
\usepackage{booktabs, makecell}
\usepackage{colortbl}
\definecolor{mygray}{gray}{.9}
\usepackage{pifont}% http://ctan.org/pkg/pifont

\def\BibTeX{{\rm B\kern-.05em{\sc i\kern-.025em b}\kern-.08em
    T\kern-.1667em\lower.7ex\hbox{E}\kern-.125emX}}
\begin{document}

\title{GATE: \textbf{G}raph CC\textbf{A} for \textbf{T}emporal S\textbf{E}lf-supervised Learning for Label-efficient fMRI Analysis}

\author{Liang Peng, Nan Wang, Jie Xu, Xiaofeng Zhu, \IEEEmembership{Senior Member, IEEE} and Xiaoxiao Li \IEEEmembership{Member, IEEE}
\thanks{L. Peng and N. Wang are contributed equally to this paper.}
\thanks{L. Peng and J. Xu are with Center for Future Media and Department of Computer Science and Engineering, University of Electronic Science and Technology of China, Chengdu 610054, China.}
\thanks{N. Wang is with School of Computer Science and Engineering, East China Normal University, Shanghai, 200062, China, and the University of British Columbia, Vancouver, BC V6T 1Z4  Canada}
\thanks{X. Zhu is with 
School of Computer Science and Engineering, University of Electronic Science and Technology of China, Chengdu 610056, China, and also with the Shenzhen Institute for Advanced Study,
University of Electronic Science and Technology of China, Shenzhen, China.}
\thanks{X. Li is with Electrical and Computer Engineering, the University of British Columbia, Vancouver, BC V6T 1Z4  Canada (e-mail:xiaoxiao.li@ece.ubc.ca).}
}
\maketitle

\begin{abstract}

%Disease prediction on a population can be naturally modeled using graphs that capture the similarity between individuals represented as nodes on the graph. 
In this work, we focus on the challenging task, neuro-disease classification, using functional magnetic resonance imaging (fMRI).
In population graph-based disease analysis, graph convolutional neural networks (GCNs) have achieved remarkable success.
However, these achievements are inseparable from abundant labeled data and sensitive to spurious signals.
To improve fMRI representation learning and classification under a label-efficient setting, we propose a novel and theory-driven self-supervised learning (SSL) framework on GCNs, namely \textbf{G}raph CC\textbf{A} for \textbf{T}emporal s\textbf{E}lf-supervised learning on fMRI analysis (\ours).
Concretely, it is demanding to design a suitable and effective SSL strategy to extract formation and robust features for fMRI. To this end, we investigate several new graph augmentation strategies from fMRI dynamic functional connectives (FC) for SSL training. Further, we leverage canonical-correlation analysis (CCA) on different temporal embeddings and present the theoretical implications. 
Consequently, this yields a novel two-step GCN learning procedure comprised of (i) SSL on an unlabeled fMRI population graph and (ii) fine-tuning on a small labeled fMRI dataset for a classification task. 
Our method is tested on two independent fMRI datasets,  demonstrating superior performance on autism and dementia diagnosis. 
Our code is available at \textit{https://github.com/LarryUESTC/GATE}.
\end{abstract}


\section{Introduction}
\label{sec_introduction}

As the brain Functional Connectivity (FC) derived from functional Magnetic Resonance Imaging (fMRI) can capture abnormal brain functional activity \cite{10.1227NEU}, it has been widely applied in disease diagnosis.
Recently, to quantify changes in brain connectivity over time, many researches apply the sliding window method to extract dynamic FC matrices from Blood-Oxygen-Level Dependent (BOLD) signals~\cite{DynamicFC}. FC representations will be further used for disease diagnosis through machine learning methods.

With the success of machine learning, deep learning methods have made breakthroughs in fMRI classification, such as Convolutional Neural Networks (CNNs) \cite{li2018brain}, Recurrent Neural Networks (RNNs) \cite{dvornek2019jointly}, and Graph Convolutional neural Networks (GCNs) \cite{parisot2018disease}.
Indeed, GCNs are widely used for disease diagnosis \cite{parisot2018disease,zhou2021contrast}, due to the biomarkers across all subjects are critical for recognizing the common patterns associated with diseases.
Previous studies have explored the development of GCNs for various prediction tasks on fMRI data \cite{parisot2018disease,ghorbani2022ra,li2018brain,bessadok2021graph}. For example, Parisot \etal~\cite{parisot2018disease} apply vanilla GCN for supervised disease prediction with fMRI data. However, these studies depend heavily on large-scale labeled$/$annotated data to get promising results.
Conversely, this condition is usually not satisfied in clinical practice due to the costly and complex annotation process. For example, the annotation of autism samples requires doctor scoring on child’s behavior and developmental history (typically from a few months to years) following a complex protocol~\cite{hyman2020executive}.

To leverage unlabeled data and assistant representation learning on the label-efficient data \cite{luo2017label} (\ie a small portion of labeled data), Self-Supervised Learning (SSL) has emerged as a powerful approach of unsupervised learning~\cite{sun2021context}. 
Finding the suitable SSL strategy for fMRI signals is essential and existing SSL strategies can be generally divided into three categories: contrastive-based SSL, reconstruction-based SSL, and similarity-based SSL.
Contrastive-based SSL~\cite{chen2020simple,NEURIPS2020_63c3ddcc,zhou2021contrast} requires selecting diverse negative samples to form contrastive loss (\eg InfoNCE loss \cite{oord2018representation} and triple loss \cite{mo2022simple}), which is difficult for disease classification using fMRI due to the limited number of samples and a small number of classes.
Reconstruction-based SSL \cite{he2021masked,qiu2022vgaer} contains an encoder-decoder structure to reconstruct input, but it is not practical as fitting the low signal-to-noise ratio fMRI features may overfit to spurious features~\cite{he2021masked}. 
Therefore, we focus on similarity-based SSL~\cite{grill2020bootstrap,he2020momentum} which forces the similarity of multiple views of the same data (\eg data and its augmentation) and can provide the best practical value to assist node classification on the graph with unlabeled fMRI data.

\begin{figure}
    \centering
    \includegraphics[width=0.98\linewidth]{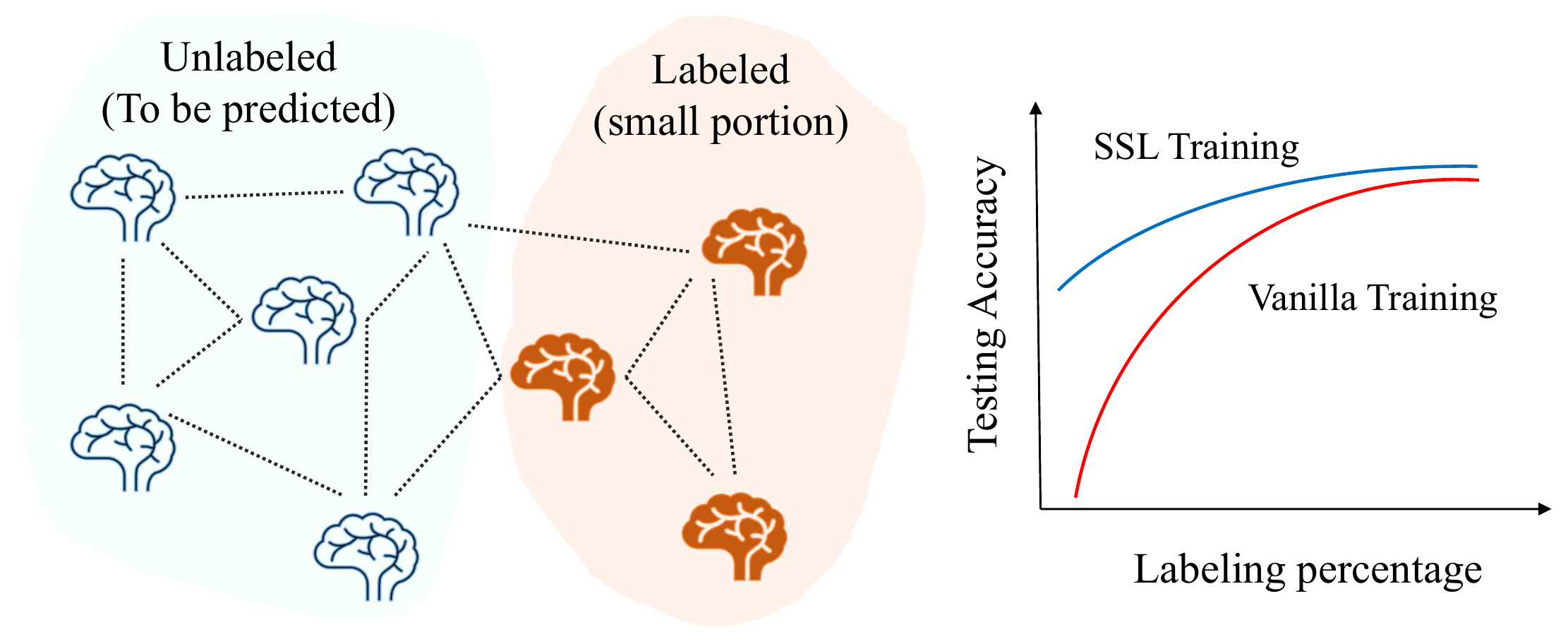}
    \caption{Illustration. This work focuses on leveraging self-supervised learning (SSL) strategy to improve fMRI prediction performance on the population graph with limited labels (left), where each node indicate a subject. Empirical studies (right, also see Sec.~\ref{sec_experiments}) have shown SSL achieves higher testing accuracy than vanilla training when a small percent of data are labeled on the graph.}
    \label{fig:intro}
\end{figure}

However, the promising similarity-based SSL strategy brings the two unique challenges for GCN-based fMRI analysis, which have been un-explored so far: 

\noindent \textbf{Challenge 1: How to find suitable augmentations for fMRI analysis to generate different views from one BOLD signal?\\} 
\revision{
Modern SSL methods rely heavily on applying various data augmentations to create different views of the sample \cite{ni2021close}. As shown by Tian \etal \cite{tian2020makes}, SSL methods take advantage of maximizing mutual information between the different views of the sample to generate discriminative representations/embeddings. 
However, not all data augmentations have a positive effect on SSL. As suggested by Wang \etal \cite{wang2022chaos}, a suitable/helpful data augmentation for SSL should satisfy certain principles. Moreover, as suggested by~\cite{chen2020self}, the augmentations should reduce the correlation between the spurious features and target labels.
}\revisionfoot{R1}

\noindent \textbf{Challenge 2: How to design the corresponding consistency loss for SSL training on fMRI analysis?\\} 
The consistency among correlated signals should be maximized. As a classical multi-variate analysis method, Canonical Correlation Analysis (CCA)~\cite{thompson1984canonical} has been widely used in fMRI analysis \cite{friman2003adaptive} and aims to maximize the correlation between two representations.
% We focus on Canonical Correlation Analysis (CCA)~\cite{thompson1984canonical}, which is a classical multi-variate analysis method aiming to maximize the correlation between two representations and has been widely used in fMRI analysis \cite{friman2003adaptive}.

To address aforementioned challenges, we propose \textbf{G}raph CC\textbf{A} for \textbf{T}emporal s\textbf{E}lf-supervised learning on fMRI analysis (\ours{}), to achieve promising results 
 with the assistance of pre-training on unlabeled fMRI data and fine-tuning on small labeled data (see Fig. \ref{fig:intro}). 
Specifically, as shown in Fig. \ref{fig_framework}, we first develop an augmentation strategy for fMRI analysis that generates two related views from the BOLD signals. 
% The main motivation for the augmentation strategy is that it is important for SSL to capture the critical information related to diseases from two views.
% The main motivation of this augmentation strategy is that it makes the SSL be able to capture the critical information related to diseases from two diverse views.
\revisionR{The main motivation is to capture the critical information related to diseases from two diverse views using SSL techniques.}
% With these two views, we conduct a GCN encoder to obtain embedding matrices of two views, because GCN can extract associations between subjects.
With these two diverse views, we conduct a GCN encoder to obtain their embedding matrices which extract the associations among subjects.
Furthermore, the CCA-based objective function is performed to maximize the correlation of the representations from two views.  
The novelties and contributions of this work could be summarized as follows:
\begin{itemize}
\item We propose a novel SSL method (\ours{}) on fMRI data, which is an effective and versatile framework to solve the problem of learning on label-efficient datasets. This could bring GCN-based methods from research to clinical applications where labels are difficult to be collected.

\item The proposed \ours{} can tackle the spurious factors in dynamic FC analysis by developing a GCN-based CCA regularization with the designed multi-view temporal augmentation strategy on BOLD signals.

\item We conduct a theoretical discussion to support our motivation and prove the critical implication of how \ours{} assists learning on label-efficient data. 

\item The comprehensive comparison experiments demonstrate that \ours{} achieves state-of-the-art performance under the label-efficient setting. We also conduct extensive ablation experiments to discuss key components of our design and algorithm.

\end{itemize}

\section{Related work}
\begin{figure*}[!t]
\centering
\includegraphics[scale=0.45]{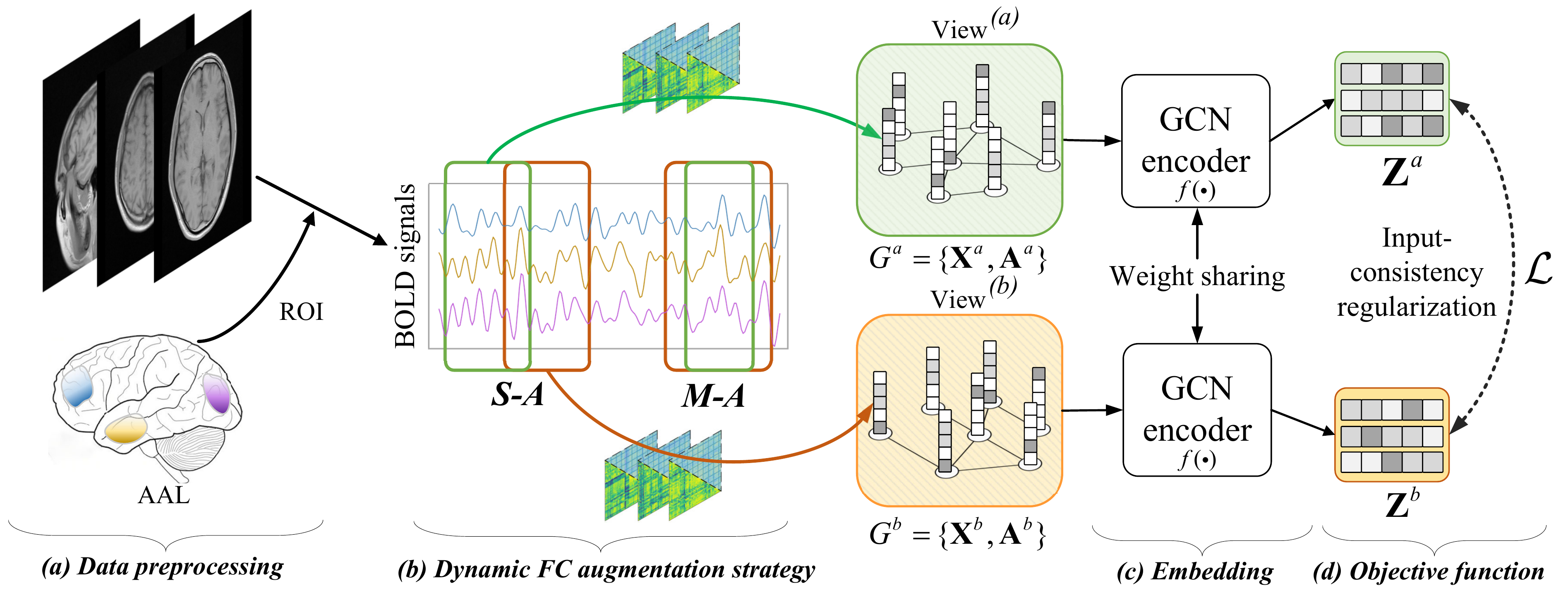}
\caption{The flowchart of the proposed method (\ours{}) for SSL-based dynamic FC representation learning. 
(a) Obtain BOLD signals by data preprocessing with AAL template.
(b) Two views $G^{a}$ and $G^{b}$ are randomly generated from the BOLD signals by our augmentation strategy (\ie \textbf{\textit{S-A}} and \textbf{\textit{M-A}}).
(c) \ours{} employs the GCN encoder to obtain embeddings  $\mathbf{Z}^{a}$ and $\mathbf{Z}^{b}$ of these two views.
%Motivated by CCA, 
(d) The consistency of the embeddings is regularized by optimizing CCA-based loss (\ie Eq.~\eqref{eq:loss}). }
\vspace{-3mm}
%by maximizing the correlation of the embeddings of the two views.}
\label{fig_framework}
\end{figure*}
\subsection{Disease Prediction on fMRI data} \label{sec_related_a}

Medical imaging techniques such as Computed Tomography (CT), Magnetic Resonance Imaging (MRI), Positron Emission Tomography (PET), and X-ray have been employed for diagnosis and early detection of diseases. For example, Kam \etal utilized multiple convolutional neural networks  to conduct early mild cognitive impairment diagnosis  \cite{kam2018novel}.
In neuroscience, MRI is widely used to study neurological disorders, which can be separated between structural MRI and functional MRI depending on the detected type of connectivity.
In structural MRI, nodes are defined by anatomical connections between regions or brain tissues, while edges are constructed by the topology between them. For instance, Yao \etal presented a triplet GNN model for disease diagnosis on structural MRI data \cite{yao2021mutual}.
In functional MRI (fMRI), nodes represent functional regions of the brain, while edges are constructed by correlations between their activities. For example, Wang \etal proposed a similarity-driven multi-view model  to conduct autism spectrum disease diagnosis on fMRI data \cite{wang2022multi}.
In fact, it has been a prevalent way to exploit the functional organization of human brain by analyzing functional fMRI, as shifts in human attention (\eg a stimulus or cognitive task) are associated with systematic changes in the activity of functional areas of brains \cite{heeger2002does,huang2020attention}.

% Node classification methods are the common way to conduct the disease diagnosis, and its goal is to search mechanistic details of behavior and activity in brain regions.
% Huang \etal proposed a attention-diffusion-bilinear neural model to conduct brain network analysis via utilizing fMRI and diffusion tensor imaging data \cite{huang2020attention}.
% Zhang \etal proposed a residual CNN model with self-attention way to conduct Alzheimer’s Disease diagnosis \cite{zhang2021explainable}.
There is growing evidence that functional connectivity may show dynamic changes over a short period of time ~\cite{zhang2016neural,mao2017gender,mertzios2019sliding}.
Previous works have demonstrated that the using of sliding windows is beneficial to conduct dynamic connectivity analysis on fMRI data since the resting-state reflects different physiological states, such as wandering and muscle fatigue.
For example, 
Wang \etal first used sliding windows to divide the rs-fMRI time series into multiple  segments and then proposed a spatio-temporal convolutional-recursive neural network \cite{wang2019spatial}.
Mao \etal analyzed the dynamic functional differences between men and women based on sliding window measures and community detection  \cite{mao2017gender}.
Yao \etal proposed a temporal adaptive graph convolutional network to exploit the topological information and to model multi-level semantic information within the entire time series signals \cite{yao2020temporal}.
\revision{
Yu \etal \cite{yu2022disentangling} presented a Transformer-based method to infer functional networks in both spatial and temporal space in a self-supervised manner.
}

\subsection{ GCNs for disease prediction on fMRI data}

Deep learning has revolutionized computer-aided diagnosis in the past decade \cite{shen2017deep,fan2019birnet}. 
On this basis, deep GCNs takes the advantages of the relationship between nodes to find the common patterns/biomarkers, and there has been an increasing focus on GCNs in neuroscience \cite{li2018brain}.
Previous GCN-based methods on fMRI data can be categorized into two subgroups depending on the definition of nodes in the graph, \ie population graph-based models and brain region graph based models \cite{bessadok2021graph}. 
Also,  these two subgroups correspond to two separate tasks for fMRI analysis, \ie population graph-based models for node classification task and brain region graph based models for graph classification tasks.

In population graph-based models, the nodes in the graph denote the subjects and edges represent the similarity between subjects.   
For example, Parisot~\etal~\cite{parisot2018disease} exploits the GCN and involves representing populations as a sparse graph in which nodes are associated with imaging features, and edge weights are constructed  from phenotype information (\eg age, gene, and sex of the subjects). 
Kazi~\etal~\cite{kazi2019inceptiongcn} propose a GCN model with multiple filter kernel sizes on a population graph. 

In brain region graph based models, the nodes denote anatomical brain regions and the edges represent functional or structural connectives among these brain regions. 
For example, Li~\etal~\cite{li2018brain} conduct an interpretable GCN model on brain region graph to understand which brain regions are related to a specific neurological disorder.
Xing~\etal~~\cite{xing2019dynamic} consider a GCN model with long short term memory model on brain region graph for disease diagnosis.

Although previous GCN methods achieved promising results on both population graph and brain region graph, existing methods are still limited by large labeled data. Clearly,  designing a GCN method for limited labeled fMRI data is an important yet unsolved challenge.
% However, two issues remain in the current GCN models.
% First, the effectiveness of GCN highly relies on the quality of the graph. When a low-quality graph is input, \ie a graph with missing component~\cite{taguchi2021graph, zhang2021subgraph}, the classification performance is affected. To address the problem of graphs containing missing features, the most popular strategy is to estimate and fill in the unknown values before~\cite{yi2019not, luo2018multivariate} or jointly with ~\cite{zhang2021subgraph} applying GCN. Different from the existing methods, we combine edge augmentation strategies to generate fused graphs.
% Second, many GCNs solely rely on a single modality, \ie imaging data \cite{abraham2017deriving}, to explore the similarity between two nodes. Consequently, they fail to comprehensively capture interactions and similarities between subjects or their individual scans~\cite{zhu2022interpretable}. To tackle this issue, we explore imaging and non-imaging data to represent the population graph. 

\subsection{Self-supervised learning}
\label{sec:relate_ssl}
As an powerful unsupervised representation learning method, SSL has recently been developed in graph learning domain. 
Existing SSL methods can be broadly classified into three groups, \ie contrastive-based SSL, reconstruction-based SSL and similarity-based SSL.

The contrastive-based SSL methods enhance the similarity between the representations from two views (\eg global representation and local representation) by manually constructing positive and negative sample pairs. 
For example, Deep Graph Infomax (DGI) \cite{DGI} employs InforNCE \cite{oord2018representation} contrastive loss function to contrast the local node representations and the global graph representations. MVGRL \cite{hassani2020contrastive} proposes a contrastive multi-view representation learning method by contrasting local and global embeddings from two views. However, contrastive-based SSL methods rely on negative samples, which is not suitable with limited number of samples and small number of classes.

Recently, reconstruction-based SSL methods learn a representation by conducting reconstruction-based pretext tasks (\eg image inpainting and  recovering color channels).
For example, He~\etal~\cite{he2021masked} develop an SSL encoder-decoder model to reconstruct the original image from the latent representation and mask patches.
Qiu~\etal~~\cite{qiu2022vgaer} propose a graph encoder-decoder model to reconstruct the  relationships between nodes based on representation similarity. 
However, the reconstruction-based SSL needs to reconstruct input context (low dimensional features are transformed into high dimensional features), which is impractical due to fitting the low signal-to-noise ratio fMRI features may result in overfitting to spurious features.

More recently, similarity-based SSL takes the advantage of the coupling between multiple views of the same data (\ie input context and its augmentation) to learn a good representation without labeled data.
For example, Grill~\etal~\cite{grill2020bootstrap} learn representations by encoding two augmented views through two distinct graph encoders. 
Although this technique avoids the limitations of selecting negative samples in the contrasitive-based SSL method and reconstruction loss in the reconstruction-based SSL method, designing a suitable strategy for generation multiple coupled views for fMRI data is highly needed.

\section{Method}
\label{sec_method}

As shown in Fig. \ref{fig_framework}, we give an overview of our similarity-based SSL method.
Specifically, the key components of the proposed \ours{} include three parts: (a) Dynamic FC augmentation (Sec. \ref{sec_augmentation}); (b) GCN encoder (Sec. \ref{sec:embed}); (c) Objective function (Sec. \ref{sec:obj}),
and our method employs a two-step training procedure under label-efficient settings, \ie step 1: unsupervised pre-training on unlabeled fMRI data, and step 2: fine-tuning on a small portion of labeled data for a specific prediction task. This section focuses on illustrating the novel SSL strategy used in step 1 and finally introduces step 2.

\subsection{Multi-view fMRI dynamic functional connectivity generation}
\label{sec_augmentation}

The  key  idea  of  \ours{} is to ensure the consistency of two augmented views of an input content. Specifically, by adopting similarity-based SSL, spurious factors can be mitigated by augmenting its variations specifically. \revisionR{For example, if we train a model to identify ``apple'' from peach, these two views can be “red apple” and ``yellow apple''. Both views contain the main characteristics (\ie shape) of ``apple'' but vary on the spurious features (\ie color).}
% For example, if we train a model to identify ``apple'', these two views can be ``red apple'' and ``yellow apple'', since both views contain the main characteristics (\eg shape) of ``apple''.
% To fit for general GCN pipeline analysis for

\noindent \textbf{Dynamic functional connectivity.} To capture the temporal variability, the entire time \revisionR{fMRI} course was divided into multiple sub-segments by the sliding window method \cite{DynamicFC}. 
%We can get $K$ sub-segments by the formula: $K=\left[\frac{T-L}{s}\right]+1$, where $T$ is the length of time points, $s$ means the length of
% step and $L$ indicates the length of the sliding window. 
Denote BOLD signals $\mathbf{S}_{i} \in \mathbb{R}^{R \times T}$ with $R$ being the number of brain Regions-Of-Interests (ROIs) in fMRI images of the $i$-th subject and $T$ being the length of the entire segment.
Then, we construct the FC matrix ($\mathbf{F}_{i} \in \bR^{R \times R}$) by calculating the Pearson’s correlation~\cite{schober2018correlation} between the matched BOLD segments of the paired ROIs. Last, we flatten the upper triangle of the FC matrix to represent the fMRI features $\mathbf{x}_i$ for the $i$-th subject.  
\revision{
To construct the population graph $G =\{\mathbf{X}, \mathbf{A}\}$ for fMRI analysis~\cite{parisot2018disease}, where $\mathbf{X}$ represents subject features and $\mathbf{A}$ represents similarities between subjects, we extract temporal fMRI BOLD signals via constructing FCs as the node features, and follow Parisot \etal ~\cite{parisot2018disease} to obtain the initial graph $\mathbf{A}$ via $k$NN from subject features.}

Existing literature indicates that dynamic FC-based fMRI analysis is sensitive to the choice of the size of the sliding window $L$~\cite{savva2019assessment,shakil2016evaluation}. Therefore, we aim to impair the effect of perturbations on $L$, \ie the spurious features in dynamic FC. 
Although the brain FC is not static, a short time sub-segment and its surrounding time sub-segments should contain the same characteristics (\eg FC patterns) associated with neurological disease~\cite{8765628,zhang2016neural}. 
Thus,  we propose to augment fMRI data along its temporal domain with sliding windows  to obtain two relevant views ($G^{a} =\{\mathbf{X}^{a}, \mathbf{A}^{a}\} $ and $G^{b} = \{\mathbf{X}^{b}, \mathbf{A}^{b}\}$) by step window augmentation (\textbf{\textit{S-A}}) and multi-scale window augmentation (\textbf{\textit{M-A}}). 
% see Fig. \ref{fig_framework} (a).
% In fMRI analysis, previous studies have shown that brain functional connectivity is not static during scanning \cite{8765628,zhang2016neural}, so a short time segment and it's surrounding time segments contain the same characteristics (\eg FC patterns) associated with the neurological disease.
%To  meet  the  above  scenarios, we propose two augmentation strategies for fMRI analysis, 
% using two augmentation strategies:  

\textbf{Step window augmentation (\textbf{\textit{S-A}})}.
% \Li{In this part, we state the difference of augmentation from [52] and put your ``Hence, it is important to investigate a suitable data augmentation for SSL considering the essential role of data augmentations in SSL as well as the special properties and data structure of medical data compared with natural data.'' here.}
As shown in Fig. \ref{fig_framework} ($a$), \textbf{\textit{S-A}} takes two neighboring sliding windows as two related views  ($G^{a}$ and $G^{b}$).
% To quantify FCN changes over time, the step-wise sliding window method is one of popular methods to generate dynamic FCNs.
In this case, the raw BOLD signals $\mathbf{S}_{i}$ are divided into $M$ sub-segments $\{\mathbf{S}_{i}^{1},\dots,\mathbf{S}_{i}^{M}\}$ for the $i$-th subject, by setting the window size as $L$ and the step of the sliding window as $s$, resulting in $M=\lfloor \frac{T-L}{s} \rfloor +1$ sub-segments and the corresponding dynamic FCs.
%Then, the Pearson correlation coefficient \cite{schober2018correlation} is implied on each signal segment $\mathbf{S}_{i}^{m}$ to obtain the functional connectivity matrix $\mathbf{F}_{i}^{m} \in \mathbb{R}^{B \times B}$ ,which denote correlation between brain regions. we abstract the upper triangle part of each FC $\mathbf{F}_{i}^{m}$ as the feature vector $x_{i}^{m} \in \mathbb{R}^{D}$ of the $m$-th segments for the $i$-th subjects, obtaining $\{\mathbf{X}^{1}, \mathbf{X}^{2}, \cdots, \mathbf{X}^{M}\}$ for all subjects. Moreover, $\{\mathbf{A}^{1}, \mathbf{A}^{2}, \cdots, \mathbf{A}^{M}\}$ are constructed by $k$NN from the features or phenotype information of subjects, where $\mathbf{A}^{m} \in \mathbb{R}^{N \times N}$ implies the similarity between subjects.
In each training iteration, \textbf{\textit{S-A}} first randomly select one sub-segment (\eg $m$-th sub-segment) for the first view $G^{a} =\{\mathbf{X}^{m}, \mathbf{A}^{m}\}$, and then regard its neighboring sub-segment (\eg $(m \pm 1)$-th sub-segment) as the second view $G^{b} =\{\mathbf{X}^{m \pm 1}, \mathbf{A}^{m \pm 1}\}$.

\textbf{Multi-scale window augmentation (\textbf{\textit{M-A}})}. 
Considering that FCs generated by different window sizes contain the relevant information, \revisionR{\textbf{\textit{M-A}} takes two different scales of sliding windows as two related views.}
% takes two different scale sliding windows as two related views. 
In this case, the $\mathbf{S}_{i}$ is divided into multiple $M$ sub-segments depending on the window size $L$. For the $m$-th sub-segment ($m\in [M]$), we sample BOLD signals using two different window sizes $l_a$ and $l_b$ to calculate FCs, forming two views of graph $G^{a} =\{\mathbf{X}^{m,l_a}, \mathbf{A}^{m,l_a}\}$ and $G^{b}=\{\mathbf{X}^{m,l_b}, \mathbf{A}^{m,l_b}\}$. These two views are fed into each training iteration of \ours{}, as shown in Fig. \ref{fig_framework} ($a$).  % $\{\mathbf{X}^{1,k}, \mathbf{X}^{2,k}, \cdots, \mathbf{X}^{M,k}\}_{k=1}^{K}$, where $K$ is xxx.
%In each training iteration, \textbf{\textit{M-A}} first randomly select two different scale window size (\eg $k$-th and $v$-th window size) for the two views $G^{a} =\{\mathbf{X}^{m,k}, \mathbf{A}^{m,k}\}$ and $G^{b} =\{\mathbf{X}^{m,v}, \mathbf{A}^{m,v}\}$.

It is worth mentioning that it is also possible to jointly consider \textbf{\textit{S-A}} and \textbf{\textit{M-A}}, for example, by randomly selecting one in each training iteration. In addition, \textbf{\textit{S-A}} and \textbf{\textit{M-A}} can also be easily combined with randomly drop augmentation \cite{thakoor2021bootstrapped} which is commonly used (see Sec. \ref{sec:ablation}).
\revision{In fact, data augmentation plays an essential role in SSL, however not all data augmentations have a positive effect on SSL \cite{wang2022chaos}. Hence, it is important to investigate a suitable data augmentation for SSL considering the essential role of data augmentations in SSL as well as the special properties and data structure of medical data compared with natural data. Compared with \cite{zhang2021canonical}, our method directly operates on the raw data rather than the graph representation.}
\revision{Furthermore, our proposed augmentation methods for fMRI  to generate multi-view data is a versatile solution that can be combined with any model architectures.}
% \subsection{GCN}
% As GCN \cite{kipf2016semi} and its variants have been widely used to  capture semantic information and structural information in the graph data, in this study,  we employ  a GCN encoder $f(\cdot)$  to obtain the embeddings of each subject.  
% Concretely, with $G = \{ \mathbf{X}^{m}, \mathbf{A}^{m}\}$, the embeddings are obtained by 
% \begin{equation}   
% \mathbf{Z}^{m} = f( \mathbf{X}^{m}, \mathbf{A}^{m}),
% \label{eqGCN1}
% \end{equation}
% where the GCN encoder $f$ has several hidden layers, and each hidden layer includes  two operations, \ie feature learning and neighborhood aggregation. 
% More specifically, the GCN operation on the $l$-th hidden GCN layer is defined as:
% \begin{eqnarray}
% f^{(l)}( \mathbf{Z}^{(l)}, \mathbf{A})=\sigma (\mathbf{D}^{-\frac{1}{2}}\mathbf{A}\mathbf{D}^{-\frac{1}{2}}\mathbf{Z}^{(l)}\mathbf{\Theta}^{(l)}),
% \label{eqGCN2}
% \end{eqnarray}
% where $\mathbf{D}$ is the diagonal matrix of $\mathbf{A}$, and $\mathbf{\Theta}^{(l)}$ is a weight matrix which needs to be trained in the $l$-th layer, and $\sigma(\cdot)$ represents the function for activation operation.

\subsection{Graph embedding}
\label{sec:embed}
With the graph data from two views ($G^{a} =\{\mathbf{X}^{a}, \mathbf{A}^{a}\} $ and $G^{b} = \{\mathbf{X}^{b}, \mathbf{A}^{b}\}$), we perform an encoder to learn the patterns of the data.
Since GCN \cite{kipf2016semi} and its variants have been widely used to  capture the semantic information (\ie $\mathbf{X}$) and the structural information (\ie $\mathbf{A}$) in the graph data, in this study,  we adopt a GCN model as the encoder $f(\cdot)$ to get the embedding of all subjects in each view.
\revision{
It is worth noting that the graph encoder $f(\cdot)$ is versatile for applying other powerful graph learning encoders (\eg GAT \cite{GAT} and GIN \cite{GINConv}). For fair comparison and easily application, we apply the \revisionR{commonly} used GCN model as the encoder in this study. 
}
More specifically, the GCN operation on the $l$-th hidden layer is defined as
\begin{eqnarray}
\label{eq:gcn}
\mathbf{H}^{(l+1)}=\sigma (\mathbf{D}^{-\frac{1}{2}}\mathbf{A}\mathbf{D}^{-\frac{1}{2}}\mathbf{H}^{(l)}\mathbf{\Theta}^{(l)}),
\end{eqnarray}
where $\mathbf{D}$ is the diagonal matrix of $\mathbf{A}$, $\mathbf{H}^{(l)}$ is the output features of all subjects at the $l$-th layer, $\mathbf{\Theta}^{(l)}$ is the trained weight matrix, and $\sigma(\cdot)$ represents an activation function. 
\revisionR{The input features of Eq. (\ref{eq:gcn}) is the feature vector fMRI features (\ie $\mathbf{X}$ obtained by \textbf{\textit{S-A}} or \textbf{\textit{M-A}}).}
To make the embedding comparable, we use the same GCN encoder (weight sharing) to project two views into the same embedding space.
Consequently, we obtain the normalized embedding matrices of two views, \ie $\mathbf{Z}^{a} = f( \mathbf{X}^{a}, \mathbf{A}^{a})$ and $\mathbf{Z}^{b} = f( \mathbf{X}^{b}, \mathbf{A}^{b})$.

\subsection{Objective function}
\label{sec:obj}
\revision{To optimize the model, contrastive-based SSL methods \cite{DGI,hassani2020contrastive} apply contrastive loss on positive pairs and negative pairs, which is not suitable with limited number of samples and small number of classes. 
Reconstruction-based SSL methods~\cite{he2021masked} apply reconstruction loss (\ie MSE) to reconstruct input context, which would be impractical due to fitting the low signal-to-noise ratio fMRI features may result in overfitting to spurious features.
Therefore, we propose \ours{} to avoid sampling negative samples or reconstructing fMRI time series. Instead, we design novel augmentation methods together with the corresponding similarity loss to leverage unlabeled data, while prevent feature collapse in the embedding space.}
Once the embedding matrices of two views are computed,  the next step is to maximize the correlation of these two embedding matrices, as the CCA \revision{does}\revisionfoot{R2-Q4} (see Sec. \ref{sec:theor}).  Furthermore, we define an input-consistency regularization loss as
\begin{align}
\label{eq:loss}
\cL  =  -\frac{1}{N} \sum_{i=1}^{N} \frac{\left \langle \mathbf{z}^{a}_{i},\mathbf{z}^{b}_{i} \right \rangle }{\left \| \mathbf{z}^{a}_{i} \right \| \left \| \mathbf{z}^{b}_{i} \right \| } 
+  \gamma \sum_{v=a,b} \|(\mathbf{Z}^{v})^\top  \mathbf{Z}^{v} - \mathbf{I}\|_F^2,
\end{align}
where $\left \langle \cdot,\cdot \right \rangle$ is the dot product operator, and $\gamma$ is the trade-off coefficient.
\revision{ And $v$ is one of the view (\eg view $a$ and view $b$) and $Z^{v}$ is the embeddings matrices of this view.}\revisionfoot{R2-Q5}
In Eq. (\ref{eq:loss}), the first term can be regarded as a regularization of the embedded features, ensuring that the low-dimensional features can maintain the representation capability. The second term in Eq. (\ref{eq:loss}) ensures that the individual dimensions of the feature are uncorrelated to avoid collapsed solutions, \ie all outputs of the model are equal.

\textbf{Performing downstream task}. \revision{After the GCN encoder is trained, there is no need to keep the large graph structure in fine-tuning step for versatile and efficiency in clinical practice because many GNN encoders (\eg GCN and GAT) belong to transductive learning which means the $\mathbf{A}$ needs to be reconstructed after new samples involved.}
Thus, we \textit{fine-tune} the encoder without graph information (\ie by replacing $\mathbf{A}$ with identity matrix $\mathbf{I}$), followed by a linear layer with ELU activation function (denoted as $\psi(\cdot): \mathcal{Z} \mapsto \bR$) on  a small portion of labeled data using cross-entropy loss for a particular clinical prediction task. All sliding windows are entered into the model during inference processes. To this end, we give the pseudo code of \ours{}.
% Comparing with Eq. (\ref{eq:con}), the cosine distance is applied to measure the similarity in Eq. (\ref{eq:loss}), since the cosine distance is more  precisely than $\ell _{2}$-norm loss in our setting? (see Sec.(?)).

\begin{algorithm}[t!]
\small
\caption{The pseudo code of our proposed \ours{}}
\label{algo1}
\textbf{Input}:  %BOLD signals $ \mathbf{S} = \{\mathbf{S}_{1}, \cdots, \mathbf{S}_{N}\}$ where $\mathbf{S}_{i} \in \mathbb{R}^{R \times T}$, 
BOLD signal segments for all the $N$ subject $\{\mathbf{S}_{i}^{1},\dots,\mathbf{S}_{i}^{M}\}_{i=1}^N$, where $m\in [M]$ indexes the sliding window;
random drop function $\mathcal{R}(\cdot)$;
% Pearson’s
% correlation with flatten the upper triangle $\mathcal{P}(\cdot)$, 
% , $k$NN graph generate function $\mathcal{T}(\cdot)$, 
graph generation function $\mathcal{P}(\cdot)$, data in fine-tuning step $\mathbf{\tilde{X}}$ with label $\mathbf{\tilde{Y}}$, and fine-tuning classifier $\psi(\cdot)$.
% adjacency matrix $\mathbf{A} \in \mathbb{R}^{n \times n}$
% number of training epochs $T$, hyperparameters $\alpha$, $\beta$ and $K$,
\\
\textbf{Output}: Graph neural network encoder: $f(\cdot)$

\begin{algorithmic}[1]
% \State {\it Initialize  parameters;} 

% \State $\{\mathbf{S}^{1},\dots,\mathbf{S}^{M}\} \gets \mathbf{S}$, $\mathbf{S}_{i}^{M} = \{\mathbf{S}_{i}^{1},\dots,\mathbf{S}_{i}^{M}\}$ 
\State $\mathbf{S}^{a} \gets \mathbf{S}^{m};~~ \mathbf{S}^{b} \gets  \mathbf{S}^{m\pm1}$   \Comment{Random select one window}
\State {\it SSL training:}
\For{$ t \leftarrow 1,2, \cdots, \mathrm{Epochs}$} 
\State $\mathbf{X}^{a}, \mathbf{A}^{a} = \mathcal{P}(\mathbf{S}^{a});~~ \mathbf{X}^{b}, \mathbf{A}^{b} = \mathcal{P}(\mathbf{S}^{b})$
% \State $\mathbf{A}^{a} = \mathcal{T}(\mathbf{X}^{a});~~ \mathbf{A}^{b} = \mathcal{T}(\mathbf{X}^{b})$
\If{$\mathrm{use~random~drop}$}
\State $\mathbf{X}^{a}, \mathbf{A}^{a} = \mathcal{R}(\mathbf{X}^{a}, \mathbf{A}^{a})$
\State $\mathbf{X}^{b}, \mathbf{A}^{b} = \mathcal{R}(\mathbf{X}^{b}, \mathbf{A}^{b})$
\EndIf
\Comment{Generate fMRI features and population graphs}
\State $\mathbf{Z}^{a} = f( \mathbf{X}^{a}, \mathbf{A}^{a})$  \Comment{Obtain embeddings of view $(a)$}
\State $\mathbf{Z}^{b} = f( \mathbf{X}^{b}, \mathbf{A}^{b})$  \Comment{Obtain embeddings of view $(b)$} 
\State Loss $\gets $ Eq. (\ref{eq:loss})  \Comment{Objective function}
\State Updating $f(\cdot)$ with optimizer, \ie AdamW
\EndFor
\State {\it Fine-tuning for downstream task:}
\For{$ t \leftarrow 1,2, \cdots, \mathrm{Epochs}$} 
\State $\mathbf{P} = \psi(f(\hat{\mathbf{X}}, \mathbf{I}))$
\State Loss $\gets $ $\mathcal{L}_{ce}$
%=- \bE \left (y_{i}log(p_{i}) + (1-y_{i})log(1-p_{i}) \right )$  
\Comment{Train classifier using cross-entropy loss}
\State Updating $\psi(\cdot)$ and $f(\cdot)$ with optimizer, \ie AdamW
\EndFor
\end{algorithmic}
\end{algorithm}

\section{Theoretical motivation and analysis on CCA Loss}
\label{sec:theor}
%Many works seek to reduce the effect of spurious features in neural network training. %SSL on unlabeled data can seems a promising approach but inappropriate self-training strategy will increase the reliance on the spurious features if using entropy-based optimization~\cite{chen2020self}. 

Machine learning models may poorly generalize on testing data if the decision making is learned from the reliance on the spurious features. Adding input-consistency regularization can help mitigate the correlation between spurious features with the label and improve testing performance~\cite{chen2020self,wei2020theoretical}. CCA, matching the similarity of two representations, has been widely used for multi-view fusion and fMRI analysis~\cite{zhuang2020technical}, while leveraging CCA in SSL is under-explored.
%Recent work shows that self-training on diverse unlabeled data provably work to mitigate the correlation between spurious features with the label in the training dataset and improve testing performance~\cite{chen2020self}. 
%To improve the model generalized on small and noisy fMRI data, we are motivated to reduce the spurious effect of designing sliding windows in dynamic FC analysis, and introduce novel fMRI sliding window augmentations for SSL in GCN-based classifier. 
%While self-training is a promising approach to mitigate the correlation between spurious features with the label in the training dataset~\cite{chen2020self}, \cite{wei2020theoretical} theoretically shows that adding input-consistency regularization on multi-views of input data to self-training will achieve higher accuracy. 
Although \cite{zhang2021canonical} uses CCA as regulation, it ONLY implicitly connects view augmentation with Information Bottleneck~\cite{tishby2000information}. Here we improve the theoretical analysis by building the connection between a deep-learning based non-linear CCA and SSL. Moreover, we imply how they will affect the downstream task. 
% We restate the the input-consistency regularization used in \cite{zhang2021canonical}:
% \begin{align}
% \label{eq:con}
%     \min_{f}\cL_{\rm consis}  = \|f(G^{a}) -  f(G^{b}) \|_F^2 +  \lambda\sum_{i=1}^2\|f(X_i)^\top  f(X_i) - I\|_F^2,
% \end{align}
% and 
The general objective of deep CCA is expressed as 
\begin{align}
\label{eq:gencca}
    \max_{f}\cL_{\rm CCA} \coloneqq \bE_{G^{a},G^{b}}[f(G^{a})^\top f(G^{b})], 
    \\{\text{ s.t. }} \Sigma_{f(G^{a}),f(G^{a})} = \Sigma_{f(G^{b}),f(G^{b})}= \mathbf{I}, \nonumber
\end{align}
%\Li{I may rewind to the previous version, stating two independent encoders $f$ and $\eta$.}
where 
%$\lambda$ is a non-negative trade-off hyperparameter, 
$f$ is a normalized non-linear embedding: $\mathcal{G} \mapsto \mathcal{Z}$ and $f \gets \frac{f}{\|f\|_2}$ w.r.t. two views $G^{a}$ and $G^{b}$,  $\Sigma$ is the covariance matrix that $\Sigma_{f,f}=\bE_{G^{a}}[f(G^{a})f(G^{a})]$. Here, we state the identifiability between input consistent SSL regulation and non-linear CCA optimization.
\begin{theorem}[CCA loss]\label{th:identical}
Considering the optimization of~Eq.~\eqref{eq:loss} and  Eq.~\eqref{eq:gencca}, Eq.~\eqref{eq:loss} is the dual problem of Eq.~\eqref{eq:gencca} in the form of Lagrangian and satisfies Karush-Kuhn-Tucker (KKT) conditions~\cite{ghojogh2021kkt}.
\end{theorem}

Then, we build the connection between CCA regularization and bounding the generalization error on the downstream tasks.  Let variables $G^a$ and $G^b$ be two views of a data, \ie $G^a$ is one sample and $G^b$ is its variant. Denote the representation operation $\T$, low rank approximation operator  $\R$ and $h: \mathcal{X} \mapsto \bR$. Conduct SVD on $\T$: find $k$ orthonormal vectors $U = [u_1, \dots, u_k]$, $V = [v_1, \dots, v_k]$ and scalars $s = \sigma_1, \dots, \sigma_k \in \bR$ that minimizes:
\begin{align*}
\cL_{U,V,S} \coloneqq & \max_{\|h\|_{L^2(G^b)}=1}\|\T h - \T_kh\|_{L^2(G^a)}, \\
\T_k h \coloneqq & \sum_{i=1}^k \sigma_i \langle v_i,h \rangle _{L^2(G^b)}u_i = f^\top \bw,
\end{align*}
where $ \bw = \langle v_i,h \rangle$. We state the general theorem of non-linear CCA on the approximation error as follows. %In our context of SSL $(\T h) \coloneqq \bE_{G^a}[h(G^b)|G^a)]$. Denote the low rank approximation operator  $(\R h) \coloneqq  \bE_Y\left[ \bE_{G^b}[h(G^b)|Y]|G^a\right]$.
\begin{lemma}[General theorem for non-linear CCA~\cite{lee2021predicting}] \label{lemma:2} Let $f$ be the solution of Eq~\eqref{eq:gencca} and $h^*$ be the optimal function to predict $\mathbf{Y}$, the one-hot encoder of label $Y$. Denote $\sigma_i \coloneqq \bE_{G^a,G^b}[f_i(G^a)f_i(G^b)].$ The the approximation error of $f$ satisfies
\begin{align}
    \label{eq:upper}
    e_{apx}(f)& \coloneqq \min_{\bW} \bE_{G^a}[\|h^*(G^a) - \bW^\top f(G^a)\|^2] \nonumber \\ 
&\leq  2\bE_{y}[\|h^* - \R  w_{b,y}\|^2 + \|(\R - \T_k) w_{b,y} \|^2].
\end{align}
Here $(\T_k\circ h_y)(g_a) \coloneqq \sum_{i=1}^k \sigma_i \bE[f_i(G^b)h_y(G^b)]f(g_a)$, and $(\R \circ h_y )(g_a)\coloneqq \bE_{Y}[\bE_{G^b}[h_y(G^b)|Y]|G^a = g_a]$, where $h_y$ satisfies $\bE[h_y(G^b)|Y=y]=\mathbf{1}(Y=y)$.   
\end{lemma}
\begin{theorem}[Upper bound of excess risk of downstream task] \label{th:main} Let $G^{a}$ and $G^{b}$ be two views randomly generated from the same training instance and $\mathbf{Y}$ be the instance labels. Consider learning an embedding by minizing Eq.~\eqref{eq:loss} $f^*  \coloneqq \argmin \cL$ and perform downstream linear model on $\mathbf{Y}$ with $f(\cdot)$, i.e., $h(G)\coloneqq (\bW^*)^\top f(G)$, $\bW^* \gets \argmin_{\bW}\bE_{G,\mathbf{Y}}\left[\| \mathbf{Y} -\bW^\top f(G) \|^2 \right]$, for analysis simplicity $\bW \in \bR^{k\times k}$. Suppose $\mathbf{Y} = h^*(G)+N$, where $N$ is $\sigma^2$-subgaussian and $\bE[N]=0$. 
 %Denote scalars $\sigma_i \coloneqq \bE_{G^{a},G^{b}}[f_i(G^{a})f_i(G^{b})]$. 
%If Assumption~\ref{ass:indep} holds, 
If the multi-view conditional independence assumption holds,
%\footnote{The assumption is detailed in Appendix and can be easily loosed to approximate multi-view conditional independence~\cite{lee2021predicting}.} holds,
% (\ie  $G^{a} \perp G^{b} | Y$)
with probability of at least 1- $\delta$, we have the excess risk:
\begin{align}
\label{eq:risk}
    {\rm ER}_{f}(\bW) \coloneqq & \bE_{G}\left[ \|h^*(G) - \bW^\top f(G) \|_2^2 \right] \nonumber \\
    \leq & \mathcal{O} \left( \frac{\alpha}{1-\lambda} + \sigma^2 \frac{k}{n} \right),
\end{align}
where $n$ is the number of labeled samples in the downstream task, $k$ is the number of classes in $\mathbf{Y}$, $\alpha$ is the Bayes error, and $\lambda$ is the $k$-th singular value of 
the representation operation $\T$. In our context of SSL $(\T h) \coloneqq \bE[(h_y(G^{b})^\top f(G^{b})|f(G^{a})]$.
%$\bE[f(G^{a})^\top f(G^{b})]$.  

\end{theorem}

\begin{proof}
 Let variables $G^a$ and $G^b$ be two views of a data, \ie $G^a$ is one sample and $G^b$ is its variant. 
Let functions $w_{i,y}(G^a)=\mathbf{1}(h^*(G^a)=y)$ for $i \in \{a,b\}$. 
% Let $f(x)=[u_1(G^a),\dots, u_k(G^a)]: \mathcal{X} \rightarrow \bR^k$ as the representation, then the approximation error of $f$ satisfies:
% % \begin{align*}
% %     \min_{\bW} \bE[\|h^*(G^a) - \bW^T f(G^a)\|^2] \leq \sum_{y=1}^k \min_{f} 2(\|\T_k -\R\|  w_{b,y}\|^2 + \|\R  w_{b,y} - h^*\|^2
% % \end{align*}
% \begin{align*}
% \min_{\bW} \bE_{G^a}[\|h^*(G^a) - \bW^\top f(G^a)\|^2] \\ 
% \leq  \bE_{G^a}[\|h^* - \R  w_{b,y}\|^2 + \|(\R - \T_k) w_{b,y} \|^2]
% \end{align*}
Under the conditional independence, $\T = \R$ and the $(k-1)$-th maximal correlation of $G^a$ and $G^b$ is the $k$-th singular value of $\T$~\cite{makur2015efficient}.
% Following~\cite{lee2021predicting}, we make the assumption:
% \begin{assumption} \label{ass:indep} Let $G^{a}$ and $G^{b}$ be two views randomly generated from the same training instance and $Y$ be the instance label, $f:~\mathcal{G} \mapsto \mathcal{Z}$ be the embedding. We assume conditional independence $G^{a} \perp G^{b} | Y$ satisfies \footnote{Our theorem can be easily extended to approximate conditional independence.}, $\Sigma_{G_i,G_i}, \Sigma_{Y,Y}$ are full rank, and the whitened feature variable $U \coloneqq \bE[f(X)f(X)^\top]^{-1/2}f(X)$ is a $\rho^2$-subgaussian random variable.
% \end{assumption}
% \begin{lemma} \label{lemma:error} If random variable $G^{a},G^{b}\in \mathcal{G},Y\in \mathcal{Y}$ satisfy Assumption~\ref{ass:indep} and matrix $A$ is full rank with $A_{y,:} \coloneqq \bE[G]$
% \end{lemma}
We have
\begin{align*}
 & \sum_y \|\T_k w_{b,y} - w_{a,y} \|^2 \\ 
= & 1- \int_{G^a,G^b}P_{G^a,G^b}(G^a,G^b)1(h(G^a)  \neq   h(G^b)) \\ 
%= & \sum_{i=1}^2   \int_{X_i,Y}P_{X_i,Y}(x_i,y)1(h^*(x_i)\! \neq \! y) \\
 \geq & \sum_{i\in\{a,b\}} P(h^*(G^i)\neq y) \geq 1 - 2\alpha.   
\end{align*}
% $\sum_y \|\T_k w_{b,y} - w_{a,y} \|^2 \! = 1- \! 
% \int_{G^a,G^b}P_{G^a,G^b}(G^a,G^b)1(h(G^a) \! \neq \!  h(G^b)) \!  = \! \sum_{i=1}^2 \!  \int_{X_i,Y}P_{X_i,Y}(x_i,y)1(h^*(x_i)\! \neq \! y) 
% \!  \geq \! \sum_{i=1}^2 P(h^*(x_i)\neq y) \geq 1 - 2\alpha.$ 
Accordingly we obtain
\begin{align*}
%\begin{array}{l}
& \sum_y\| \R w_{b,y} - w_{a,y}\|^2 \\ 
= & \sum_y (\|\R w_{b,y}\|^2 \! +  \| w_{a,y}\|^2 
-  2 \langle w_{a,y},\R w_{b,y} \rangle) \\
\leq & 2-2(1-2\alpha)= 4\alpha.
%\end{array}
\end{align*}
Additionally,
\begin{align*}
     \sum_y \|(\T \! - \! \T_k)w_{b,y}\|^2  
   \! \leq \! \lambda^2(1 \! -\! \frac{(1-2\alpha)^2-\lambda^2}{1-\lambda^2}) 
   \! \leq \! \frac{4\alpha(1-\alpha)}{1 - \lambda^2}
\end{align*}
% $\sum_y \|(\T - \T_k)w_{b,y}\|^2 \leq \lambda^2(1 - \frac{(1-2\alpha)^2-\lambda^2}{1-\lambda^2}) \leq \frac{4\alpha(1-\alpha)}{1 - \lambda^2}$. 
Hence we have 
\begin{align*}
    & {\sum_y\|\T_k w_{b,y} - w_{a,y}\|^2} \\
    \leq & \left( \sqrt{\sum_y \| \T_k w_{b,y} - w_{a,y}\|^2} + \sqrt{\sum_y\|(\T - \T_k)w_{b,y} \|^2} \right)^2  \\
    \leq  & (2\sqrt{a}+\sqrt{\frac{4\alpha(1-\alpha)}{1-\lambda^2}} )^2\leq \frac{16\alpha}{1-\lambda},
\end{align*}
% ${\sum_y\|\T_k w_{b,y} - w_{a,y}\|^2}^2 \leq ( \sqrt{\sum_y \| \T_k w_{b,y} - w_{a,y}\|^2} + \sqrt{\sum_y\|(\T - \T_k)w_{b,y} \|^2})^2 \leq (2\sqrt{a}+\sqrt{\frac{4\alpha(1-\alpha)}{1-\lambda^2}} )^2 \leq \frac{16\alpha}{1-\lambda^2}$ 
and $\sum_y\|\T_k w_{b,y} -w_{a,y} \|^2 \leq \sum_y \|(\T - \T_k)w_{b,y}\|^2$.

Notice that we need to upper bound Eq.~\eqref{eq:upper}, where 
\begin{align*}
& \sum_{y} \left\|\left(\R-\mathcal{T}_{k}\right) w_{2, y}\right\|^{2} \\
%=& \sum_{y} \mathbb{E}_{X_{1}}\left\|\left(\R  w_{2, y}-w_{1, y}\right)+\left(w_{1, y}-\mathcal{T}_{k}  w_{2, y}\right)\right\|^{2} \\
\leq & 2\! \sum_{y} \! \left\|\R w_{2, y}\!-\!w_{1, y}\right\|^{2}\!+\!\left\|w_{1, y}\!-\!\mathcal{T}_{k}  w_{2, y}\right\|^{2}\!  \\
\leq & \frac{16 \alpha}{1-\lambda_{k}^{2}}+4 \alpha = \mathcal{O}(\frac{\alpha}{1-\lambda}).
\end{align*}
On the other hand, we have $\|h^* - \R  w_{b,y}\|^2
 \leq \langle (\bY - h^*, \bW^{*\top} f - \bW^\top f \rangle \lesssim Tr(\Sigma_{YY|G^a})(k+\log k/\delta))$. 
 Therefore, the excess risk of fitting the linear layer 
 \begin{align*}
\frac{1}{n}\|h^* - \bW^\top f\|^2_F\lesssim & \frac{1}{n} Tr(\Sigma_{YY|G^a})(k+\log k/\delta)) \\ 
= & \mathcal{O}(\sigma^2 \frac{k}{n})
\end{align*} 
\revisionR{Therefore we could easily conclude that Theorem~\ref{th:main} holds.} 
% :
% \begin{align*}
% {\rm ER}_{f}(\bW) \coloneqq & \bE_{G^a}\left[ \|h^*(G^a) - \bW^\top f(G^a) \|_2^2 \right] \\
% \leq & \mathcal{O} \left( \frac{\alpha}{1-\lambda} + \sigma^2 \frac{k}{n} \right)
% \end{align*}
% % The proof sketch is presented in Appendix.
\end{proof}

\begin{remark} \revisionR{Note that based on bound provided by  Theorem~\ref{th:main},} larger number of labeled samples $n$ in the downstream task and smaller value of $\lambda$ give a lower excess risk. 
Small value of $\lambda$ can be achieved by low rank representation of $\T$, \eg via optimizing Eq.~\eqref{eq:gencca}~\cite{lee2021predicting}. 
\end{remark}

\section{Experiments}
\label{sec_experiments}

\subsection{Experimental setup}  

\begin{table*}[!ht]
\small
\centering
\caption{\revisionR{Performance (\ie mean (std))  of all methods on two datasets under $20\%$ labeled data in the fine-tuning step.} The ``gray'' row means the results with vanilla training (semi-supervised learning or supervised learning). 
\revision{
The ``Avg'' metric means the average value of the five evaluation metrics for convenient comparison.
}
The best performance is highlighted in boldface.}
\begin{threeparttable}
\begin{spacing}{1.22}
\begin{tabular}{|r|r|ccccc|c|}
\hline
~~&Metrics & Accuracy & AUC &   Precision & Recall & F1-score  & Avg\\ 
\hline
\multirow{10}{*}{\begin{rotate}{90}{\textbf{ABIDE}}\end{rotate}}
    % &Raw feature    &55.5 (1.7) &54.6 (1.8) &57.6 (1.6) &65.9 (3.5) &61.4 (1.6) &59.0 (2.0)\\  
&Vanilla GCN \cite{parisot2018disease}&\cellcolor{mygray}59.6 (2.8) &\cellcolor{mygray}59.3 (2.7) &\cellcolor{mygray}61.0 (2.5) &\cellcolor{mygray}67.3 (5.3) &\cellcolor{mygray}64.7 (2.8) &\cellcolor{mygray}62.4 (3.2)\\ 
&GAT \cite{GAT}     &\cellcolor{mygray}61.6 (2.2)   &\cellcolor{mygray}60.8 (1.9)   &\cellcolor{mygray}62.7 (2.2)   &\cellcolor{mygray}69.7 (4.0)    &\cellcolor{mygray}66.6 (2.2)   &\cellcolor{mygray}64.3 (2.5) \\ 
&SAC-GCN \cite{song2021graph}&\cellcolor{mygray}61.8 (2.2)   &\cellcolor{mygray}61.2 (2.2) &\cellcolor{mygray}63.4 (2.3) &\cellcolor{mygray}68.6 (4.1) &\cellcolor{mygray}65.8 (2.4) &\cellcolor{mygray}64.1 (2.7)\\
&DGI \cite{DGI}           &57.4 (2.7)   &56.9 (2.4)   &59.8 (2.6)   &63.6 (5.3)   &61.6 (2.7) &59.9 (3.1) \\ 
&MVGRL \cite{hassani2020contrastive}         &58.1 (3.1)   &57.6 (2.8)   &60.5 (1.7)   &64.0 (5.2)   &62.2 (2.2) &60.5 (3.0) \\  
&BGRL \cite{thakoor2021bootstrapped}          &61.0 (2.4)   &60.3 (2.4)   &62.4 (2.0)   &69.3 (4.9)   &65.6 (2.6) &63.7 (2.9) \\  
&CCA-SSG \cite{zhang2021canonical}       &60.6 (2.3)   &60.1 (2.0)   &62.6 (2.6)   &66.5 (4.4)   &64.4 (4.9) &62.8 (3.6) \\  
& \ours  ({ours}) &\textbf{63.7} (1.8)   &\textbf{63.6} (2.1)   &\textbf{65.6} (1.5)   &\textbf{70.1} (3.7)   &\textbf{67.7} (2.5) &\textbf{66.2} (2.5) \\  
\hline
\multirow{9}{*}{\begin{rotate}{90}{\textbf{FTD}}\end{rotate}}
% &Raw feature    &60.8 (2.3) &60.6 (2.4) &63.6 (3.2) &65.3 (4.6) &64.0 (2.7) &62.7 (3.6)  
&Vanilla GCN \cite{parisot2018disease}&\cellcolor{mygray}64.5 (3.0) &\cellcolor{mygray}65.5 (2.6) &\cellcolor{mygray}65.9 (2.8) &\cellcolor{mygray}69.4 (3.6) &\cellcolor{mygray}66.8 (2.7) &\cellcolor{mygray}66.4 (2.5)\\ 
&GAT \cite{GAT}      &\cellcolor{mygray}66.4 (2.7)   &\cellcolor{mygray}66.1 (2.2) &\cellcolor{mygray}66.8 (2.2) &\cellcolor{mygray}75.7 (2.2) &\cellcolor{mygray}70.6 (2.2)    &\cellcolor{mygray}69.1 (2.3) \\ 
&SAC-GCN \cite{song2021graph}&\cellcolor{mygray}67.8  (3.8) &\cellcolor{mygray}67.5  (3.6) &\cellcolor{mygray}66.9  (2.7) &\cellcolor{mygray}71.8  (4.5) &\cellcolor{mygray}68.1  (2.9) &\cellcolor{mygray}68.5 (3.5)\\
&DGI \cite{DGI}           &62.8 (2.9)   &62.9 (2.4) &65.4 (2.8) &62.5 (7.3) &63.8 (2.0) &63.5 (3.5) \\ 
&MVGRL \cite{hassani2020contrastive}         &63.5 (3.8)   &63.4 (2.9) &65.7 (2.3) &64.0 (6.6) &64.7 (3.4) &64.3 (3.8) \\  
&BGRL \cite{thakoor2021bootstrapped}          &68.2 (2.5)   &68.5 (1.9) &67.8 (3.1) &73.6 (5.0) &70.6 (2.1) &69.7 (2.9) \\  
&CCA-SSG \cite{zhang2021canonical}       &69.7 (3.0)   &69.5 (2.4) &68.6 (2.7) &75.7 (4.5) &71.8 (2.4) &71.1 (3.0) \\   
& \ours{}  ({ours}) &\textbf{72.4} (2.5)   &\textbf{71.7} (2.0) &\textbf{70.7} (2.2) &\textbf{76.8} (4.2) &\textbf{73.3} (2.8) &\textbf{73.0} (2.7) \\ 
\hline
\end{tabular}
\end{spacing}
\end{threeparttable}
\label{tab:1}
\end{table*}

\subsubsection{Datasets} \label{sec_dataset} we conduct experiments on two fMRI datasets: Autism brain imaging data exchange (ABIDE) for health control (HC) vs. autism patient classification\footnote{http://fcon\_1000.projects.nitrc.org/indi/abide/.} and Frontotemporal dementia (FTD) for HC vs. dementia classification\footnote{https://cind.ucsf.edu/research/grants/frontotemporal-lobar-degeneration-neuroimaging-initiative-0.}.

\textbf{ABIDE} contains 1029 subjects with functional magnetic resonance imaging data from ABIDE-I and ABIDE-II datasets in this work, \revision{including 485 ASD patients and 544 healthy controls (HCs) which are nearly balanced class distribution.}\revisionfoot{R3-Q1} The registered fMRI volumes were constructed on predefined templates, using Bootstrap Analysis of Stable Cluster parcellation with 122 ROIs (BASC-122). We construct a $122 \times 122$ FC network for each subject, where each node is an ROI and the edge weight is the Pearson's correlation between the time series of BOLD signals of paired ROIs. To represent a subject, we use the upper triangle of the fully associated matrix, yielding in a 7503-dimensional feature vector.

% includes $1029$ subjects from ABIDE-I and ABIDE-II, \ie $485$ ASD patients and $544$ healthy control (HC) subjects with functional magnetic resonance imaging (fMRI) data.  
% The registered fMRI volumes are partitioned into 122 regions-of-interest (ROIs) using the Bootstrap Analysis of Stable Clusters (BASC-122) template. 
% We construct a $122 \times 122$ FC network for each subject, where each node is an ROI and the edge weight is the Pearson's correlation between the time series of BOLD signals of paired ROIs. 
% We use the upper triangle of the fully connected matrix to represent a subject, yielding a $7,503$-dimensional feature vector.

\textbf{FTD} 
% In this study, we use $0$ subjects with rs-fMRI data frome ADNI-1, ADNI-GO and ADNI 2. The rs-fMRI data are pre-processed with the pipeline of Data Processing Assistant for Resting-State fMRI (DPARSF). Specifically, we remove first 10 volumes to allow for magnetization equilibrium. Then, the detailed preprocessing steps can be seen by the following operations: (1) slice timing correction; (2) head motion correction; (3) spatial normalization to the Montreal Neurological Institute (MNI) template with 3×3×3 $mm^3$ resolution; (4) spatial smoothing using a full width at half maximum Gaussian smoothing kernel with a size of 6 mm, and (5) inear detrending and temporal band-pass filtering ($0.01$-$0.10$~Hz) for BOLD signals. Finally, the registered fMRI volumes were parcelled into $116$ ROIs according to AAL template. And we represent each node as an 6728-dimensional
% feature vector, with each element representing the Pearson’s
% correlation coefficient between this ROI and another ROI.
includes $181$ subjects \revision{consisting of 86 normal cases (HC), 95 characterized as FTD, which are nearly balanced class distribution.}\revisionfoot{R3-Q1}
Subjects are acquired on the 3.0T scanner at different centers with a gradient field strength of 80mT/m and gradient switching rate of 200mT/m/ms, using an eight-channel phased-array receiver coil.
Then, we use the DPARSF toolbox \cite{yan2010dparsf} to preprocess those data. 
The following pipeline is included: (i) slice timing correction; (ii) head motion correction; (iii) spatial normalization to the montreal neurological institute template; (iv) spatial smoothing using a full half-width Gaussian smoothing kernel; and (v) linear detrending and temporal bandpass filtering (0.01- 0.10 Hz) for BOLD signals. Finally, the registered fMRI volumes were parcellated into 116 ROIs according to the AAL template.

\revision{
\subsubsection{Graph construction} \label{sec_graph_construction}
We follow Parisot \etal \cite{parisot2018disease} to obtain the initial graph $\mathbf{A}$. 
To begin with, we extract low-dimensional and discriminative features from raw medical images. We then utilise them to build a similarity graph $\mathbf{S} \in \mathbb{R} ^{n \times n}$ where $n$ indicates the number of nodes in the population graph, intending to limit the adverse influence of high-dimensional features, such as noisy, redundant features and the dimensionality curse.
Additionally, we employ phenotype data (\eg sex, age and gene) to calculate node similarity from another angle, supplying much information to produce high-quality graphs. 
Finally, we merge the edges acquired from image-based node features and the edges generated from phenotype information to obtain the initial graph $\mathbf{A}$ by performing the Hadamard product between the similarity graph matrix $\mathbf{S}$ and the phenotypic graph matrix $\widetilde{\mathbf{S}}$ (\ie $\mathbf{A} =  \mathbf{S} \circ \widetilde{\mathbf{S}}$).
We also sparse the graph $\mathbf{A}$ by preserving the $k$ edges with the largest weights for each node and others as zeros. Equally, we add the diagonal matrix $\mathbf{I}$ to $\mathbf{A}$ (\ie $\mathbf{I} +  \mathbf{A} \rightarrow \mathbf{A}$).}
 \revisionfoot{R1}

% First, we derive low-dimensional and discriminative features from raw medical images and then use them to construct a similarity graph $\mathbf{S} \in \mathbb{R} ^{n \times n}$ where $n$ denotes the number of nodes in the population graph, aiming at reducing the adverse influence of high-dimensional features, \eg noisy and redundant features as well as the curse of dimensionality.
% Second, we calculate node similarity from another angle by utilizing phenotype data (\eg sex, age and gene), aiming at providing much information to output high-quality graphs.  
% Finally, we fuse the edges derived from image-based node features and the edges derived from phenotype information to obtain the initial graph $\mathbf{A}$ by performing the Hadamard product between the similarity graph matrix $\mathbf{S}$ and the phenotypic graph matrix $\widetilde{\mathbf{S}}$ (\ie $\mathbf{A} =  \mathbf{S} \circ \widetilde{\mathbf{S}}$). Moreover, we sparse graph $\mathbf{A}$ by keeping $k$ edges with the largest weights for each node and setting others as zeros.
% Furthermore, we add the diagonal matrix $\mathbf{I}$ to $\mathbf{A}$ (\ie $\mathbf{I} +  \mathbf{A} \rightarrow \mathbf{A}$).}

\subsubsection{Comparison methods} \label{sec_comparison_methods}
The comparison methods include three GCN methods without SSL strategies (vanilla GCN \cite{parisot2018disease}, GAT \cite{GAT}, and SAC-GCN \cite{song2021graph}), two contrastive-based graph SSL methods (DGI \cite{DGI} and MVGRL \cite{hassani2020contrastive}), 
two similarity-based graph SSL method (BGRL \cite{thakoor2021bootstrapped} and CCA-SSG \cite{zhang2021canonical}). 
\revision{We adopt the open-source codes of all comparison methods and use grid search to decide their best hyper-parameters. For a fair comparison, we use the same dynamic window strategy for the comparison methods and use grid search technology to find their best practice.}\revisionfoot{R1-Q6}
We list details of each comparison method below.
\begin{itemize}
\item \textit{Vanilla GCN} \cite{parisot2018disease} trains GCN model \cite{kipf2016semi} with cross-entropy loss on fMRI data for disease diagnosis with labeled training data. This is the baseline method of GCNs on fMRI data~\cite{parisot2018disease}.

\item \textit{GAT} \cite{GAT} is one popular GCN method with attention mechanism and is considered as the state-of-the-art method for graph learning with labeled training data.

\item \revision{\textit{SAC-GCN} \cite{song2021graph} is a population graph based method which conduct representation learning by taking into account both the functional graph and the structural graph.}

\item \textit{DGI} \cite{DGI} is a pioneer in the study of SSL with the graph structure data. As a contrastive learning method, DGI obtains good representation by manually constructing positive and negative pairs.

\item \textit{MVGRL} \cite{hassani2020contrastive} maximizes the mutual information of multiple related views inspired by DGI \cite{DGI}, which contains a more complex pipeline for graph SSL.

\item \textit{BGRL} \cite{thakoor2021bootstrapped} generalizes BYOL \cite{grill2020bootstrap} to graphs  by forming online and
target node embeddings, and gets promising results on large-scale graph representation learning.

\item \textit{CCA-SSG} \cite{zhang2021canonical} proposes a similarity-based SSL and obtains significant results. The main difference between \ours{} and CCA-SSG is that \ours{} designs a special strategy to generate two views according to the characteristics of medical data, and we improve the theoretical analysis by building the connection between a deep-learning based non-linear CCA and SSL.

\end{itemize}
\revision{For supervised learning method (\eg vanilla GCN, GAT, and SAC-GCN), we follow their original setting (transductive learning) to train the model.}
Among the above comparison methods, DGI, MVGRL, BGRL, and CCA-SSG are the state-of-the-art GCN models using SSL embeddings.
\revision{In contrastive-based models (\ie DGI and MVGRL), following their original literature and official codes, we use the corruption function and the readout function for negative construction and positive construction. In similarity-based methods (\ie BGRL and CCA-SSG), we use the random masking function on feature and graph to generate two views.}

\subsubsection{Implementation details} \label{sec_implementation}
All experiments are run on a server with 8 NVIDIA GeForce 3090 GPUs and implemented in PyTorch (vision 1.9).
All parameters are optimized by the AdamW \cite{loshchilov2017decoupled} optimizer  with $1e^{-3}$ learning rate. 
The $\gamma$ in Eq. (\ref{eq:loss}) is empirically set to $0.2$, $L$ and $s$ in \textbf{\textit{S-A}} are set to $30$ and $15$ respectively, $l_a$ and $l_a$ in \textbf{\textit{M-A}} are randomly selected from $[10,20,30,40,50]$ in each training iteration.
\revision{
We apply one graph convolutional layer following one linear layer in this study considering that the large number of  graph convolutional layer may cause the problem of over-smooth\cite{yang2020revisiting} (the problem of over-smooth is easy to happen on the small datasets).
} \revisionfoot{R1-Q1}
We apply the ELU function as a nonlinear activation for each layer. 
%  The diagnosis results of all methods are evaluated by four evaluation metrics, including \textit{Accuracy, Precision, Recall and F1-score}. For all of these metrics, a higher value means better performance. 
\revision{The \textit{training times} ($s$) on the ABIDE dataset and FTD dataset are $20.3 (\pm 0.2)$ and $14.7 (\pm 0.2)$, respectively.} \revisionfoot{R1-Q5}
In the fine-tuning step, \revision{the labelled data is randomly selected from the original set (\eg $206$ labeled samples on the ABIDE dataset and $36$ labeled samples on the FTD dataset for 20\% labeled rate in the fine-tuning step), $5$-fold cross-validation is performed and the experiments are repeated $5$ times with random seeds.}
The average performances with corresponding standard deviation (std) are reported for all methods. %More details are provided in Appendix.

\begin{figure}[!t]
\centering
\scalebox{0.85}{
\subfloat[ABIDE]{\scalebox{0.32}{\includegraphics{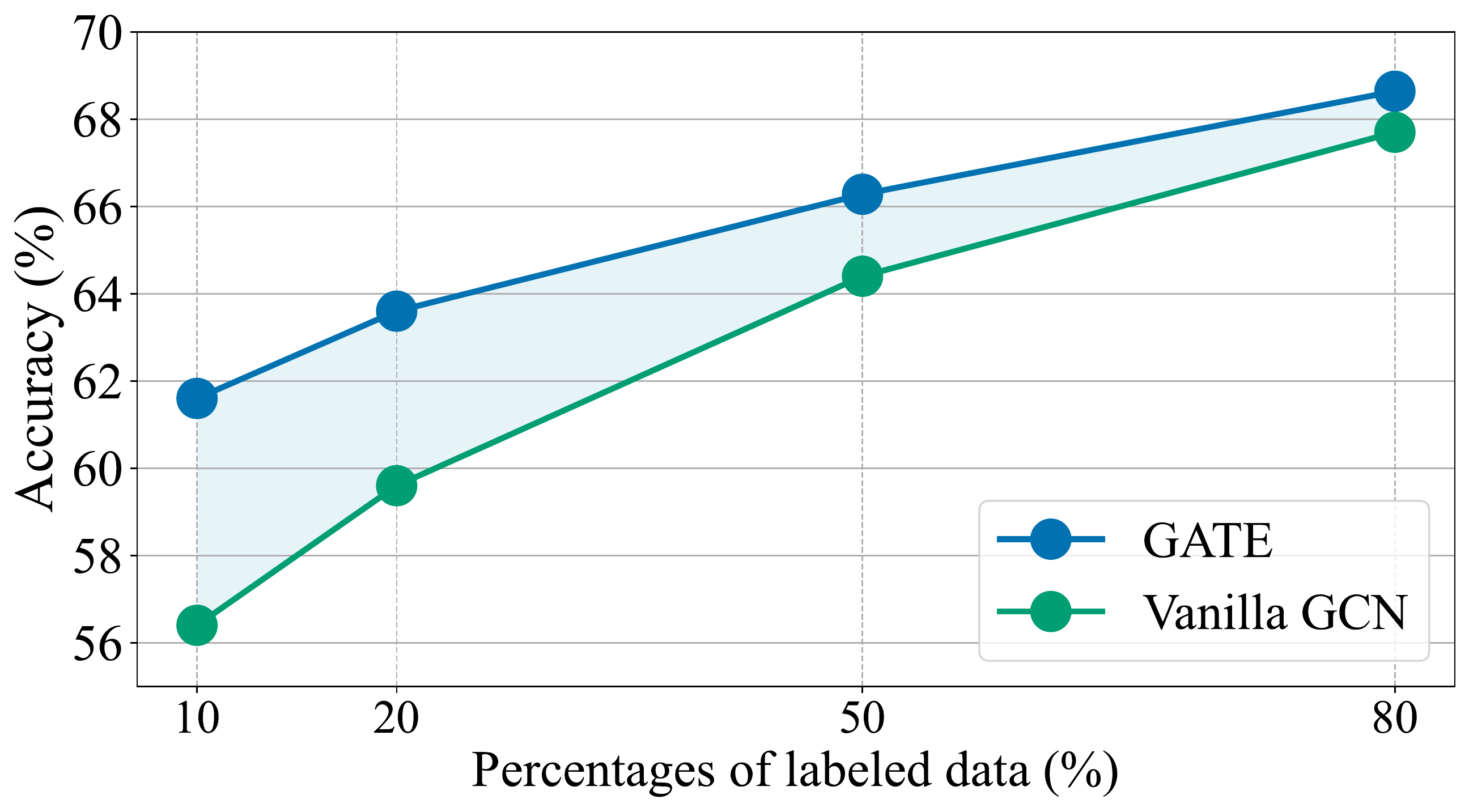}}}}\\
\scalebox{0.85}{
\subfloat[FTD]{\scalebox{0.32}{\includegraphics{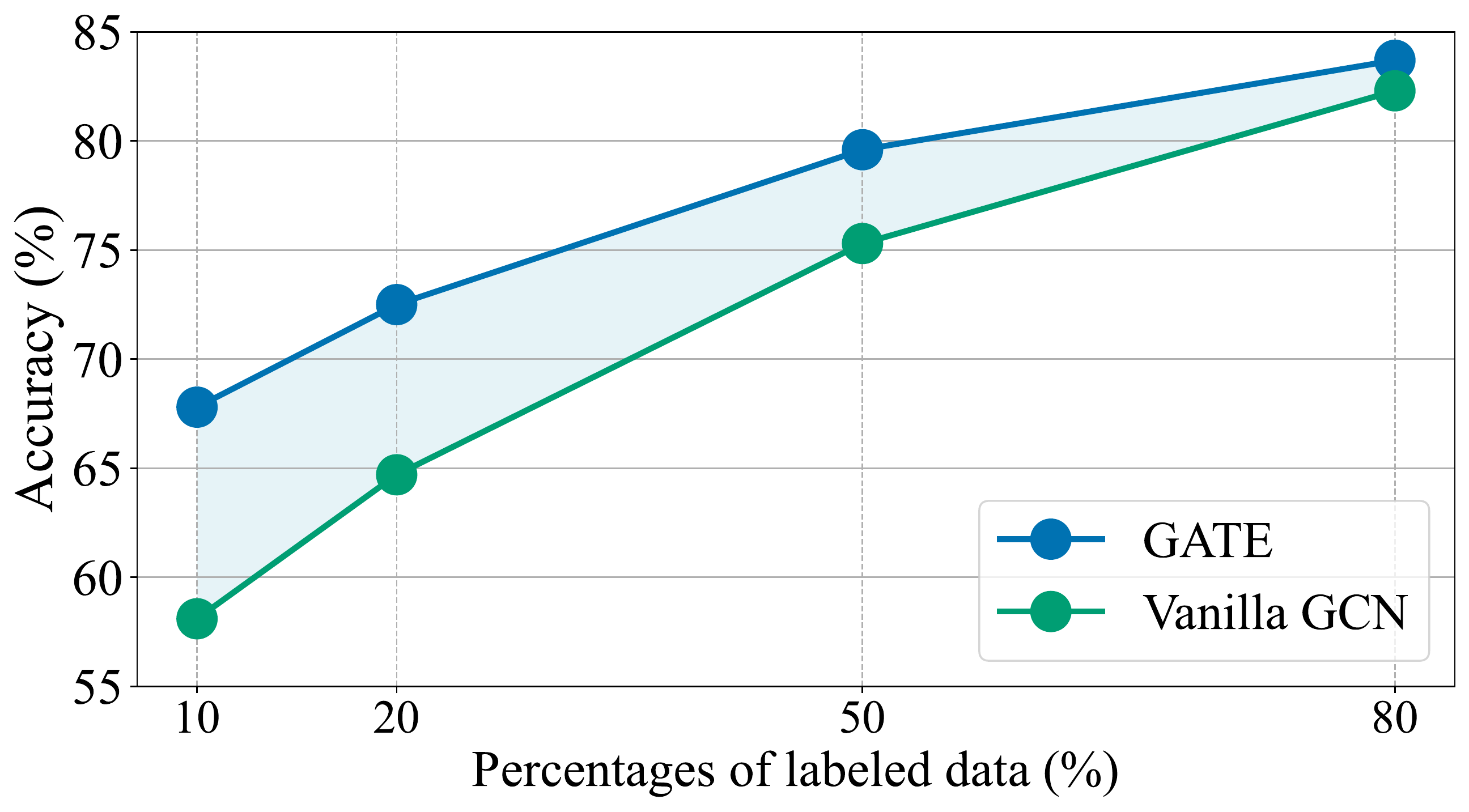}}}}
\caption{Accuracy of \ours{} and vanilla GCN at different label rates.}
\vspace{-2mm}
\label{fig2}
\end{figure}

\begin{figure}[!t]
\centering
\scalebox{0.85}{
\subfloat[ABIDE]{\scalebox{0.32}{\includegraphics{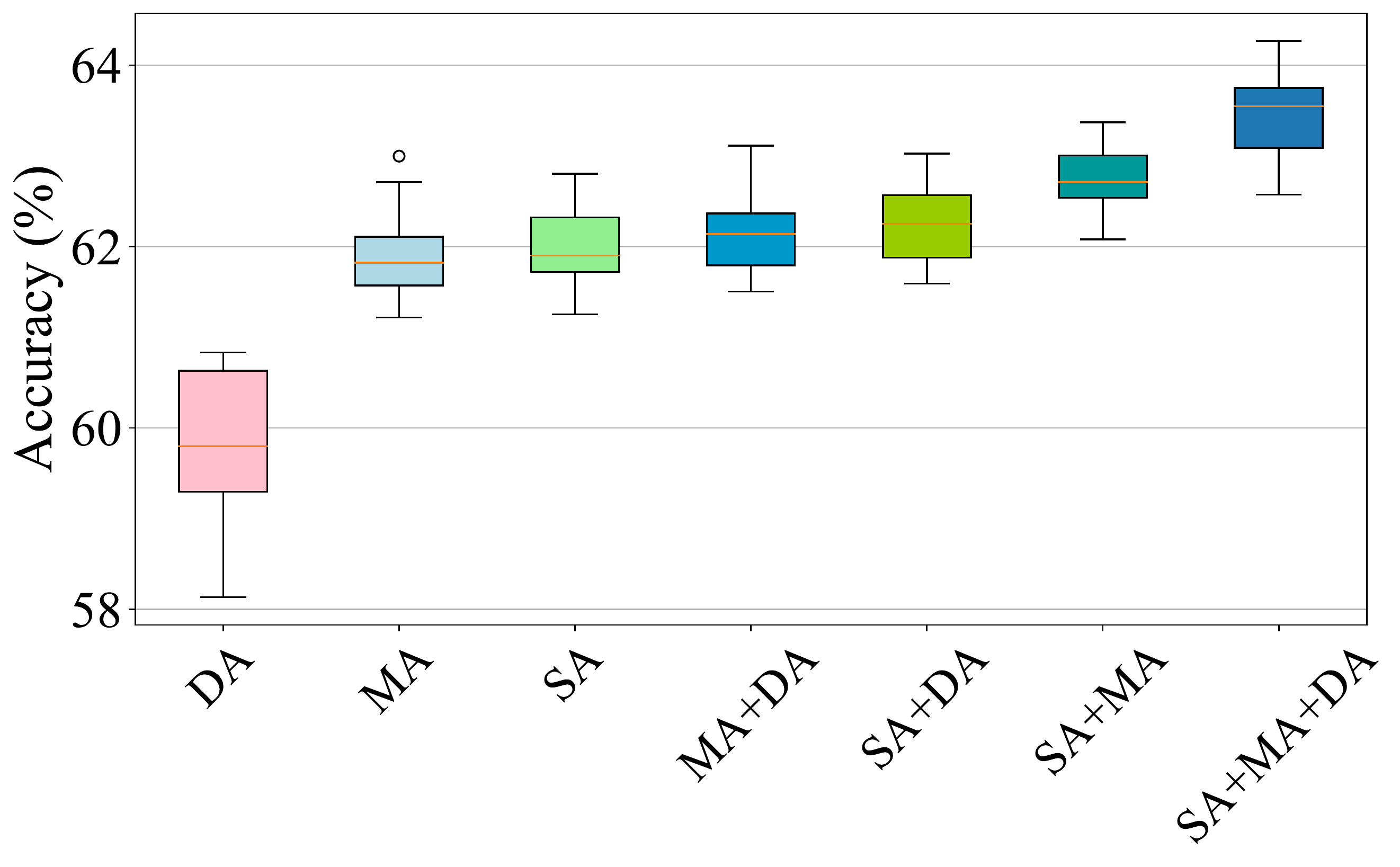}}}}\\
\scalebox{0.85}{
\subfloat[FTD]{\scalebox{0.32}{\includegraphics{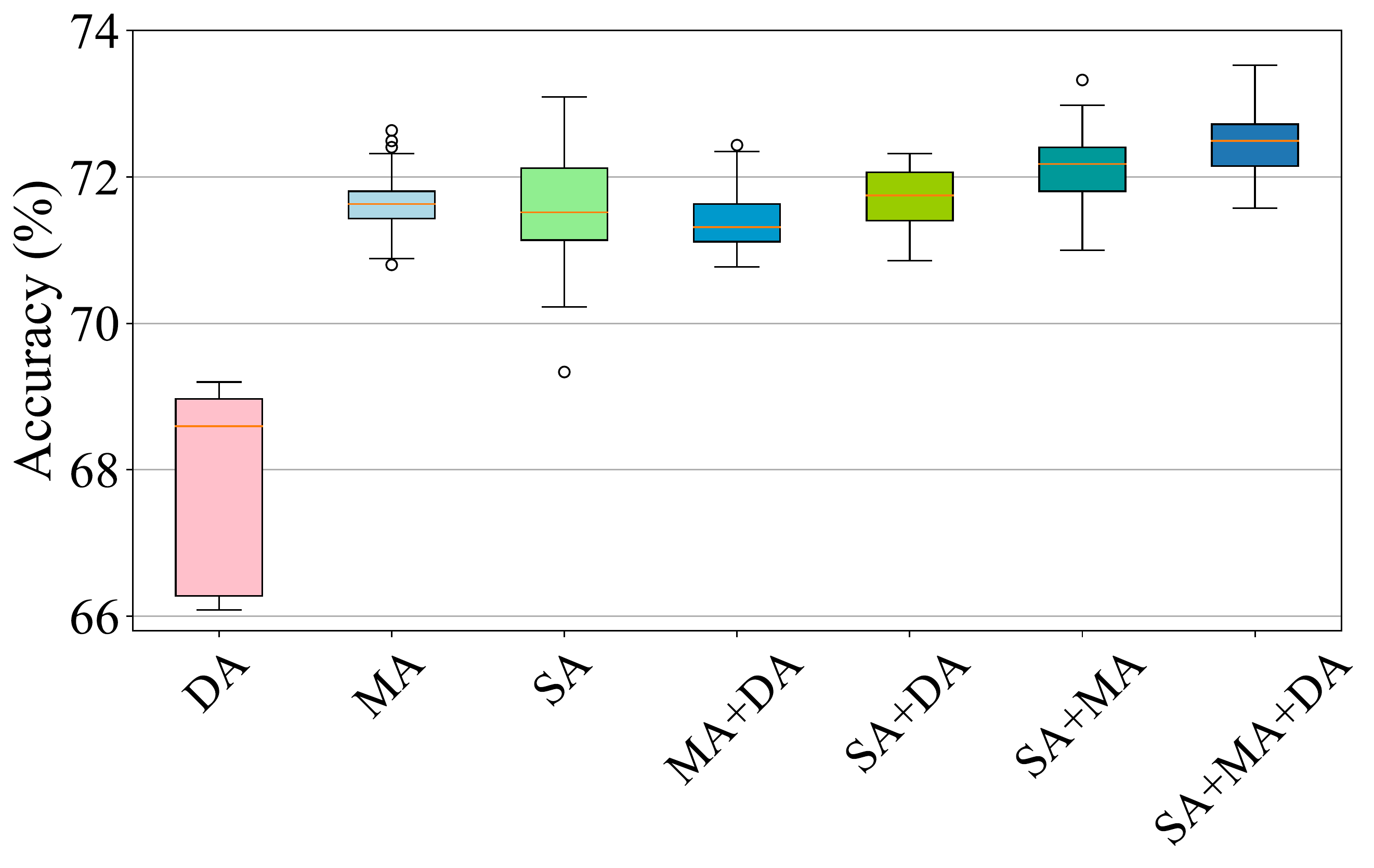}}}}
\caption{Performance of \ours{} with different combination of augmentation methods on two datasets.}
\label{fig4}
\end{figure}

\subsubsection{Performance evaluation}
% The diagnosis results of all methods are evaluated by 
The performance evaluation include five evaluation metrics, including \textit{Accuracy, Area under the ROC Curve (AUC), Precision, Recall and F1-score}. 
For all of these metrics, a higher value means better performance. 
% \revisionR{Besides, we apply the Wilcoxon signed-rank test to conduct significance testing. }
Besides, we conduct significance testing using the two sample t-test.
% We conduct significance testing using the two sample t-test.
\begin{figure}[!t]
\centering
\subfloat[ABIDE]{\scalebox{0.3}{\includegraphics{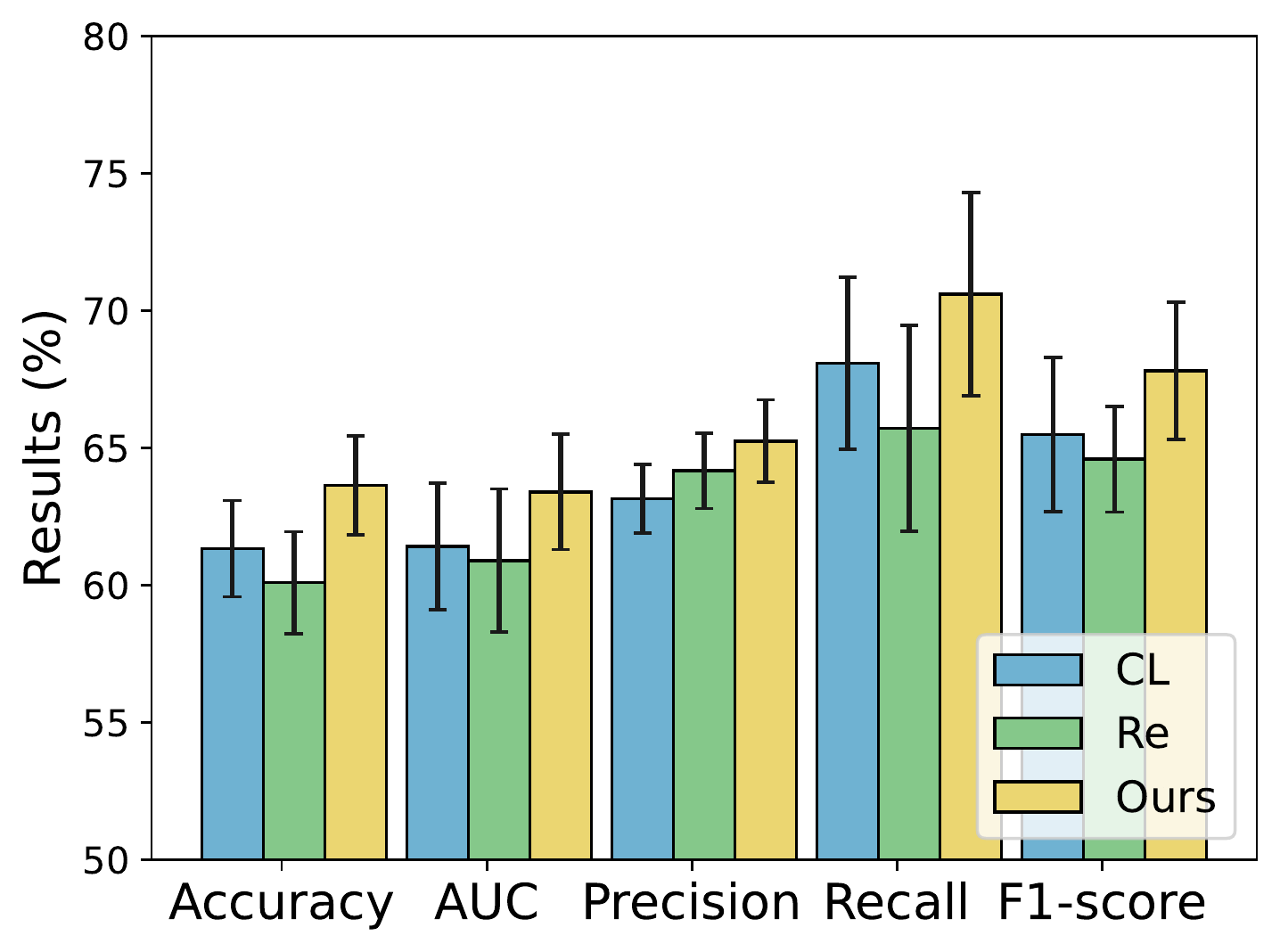}}}
\subfloat[FTD]{\scalebox{0.3}{\includegraphics{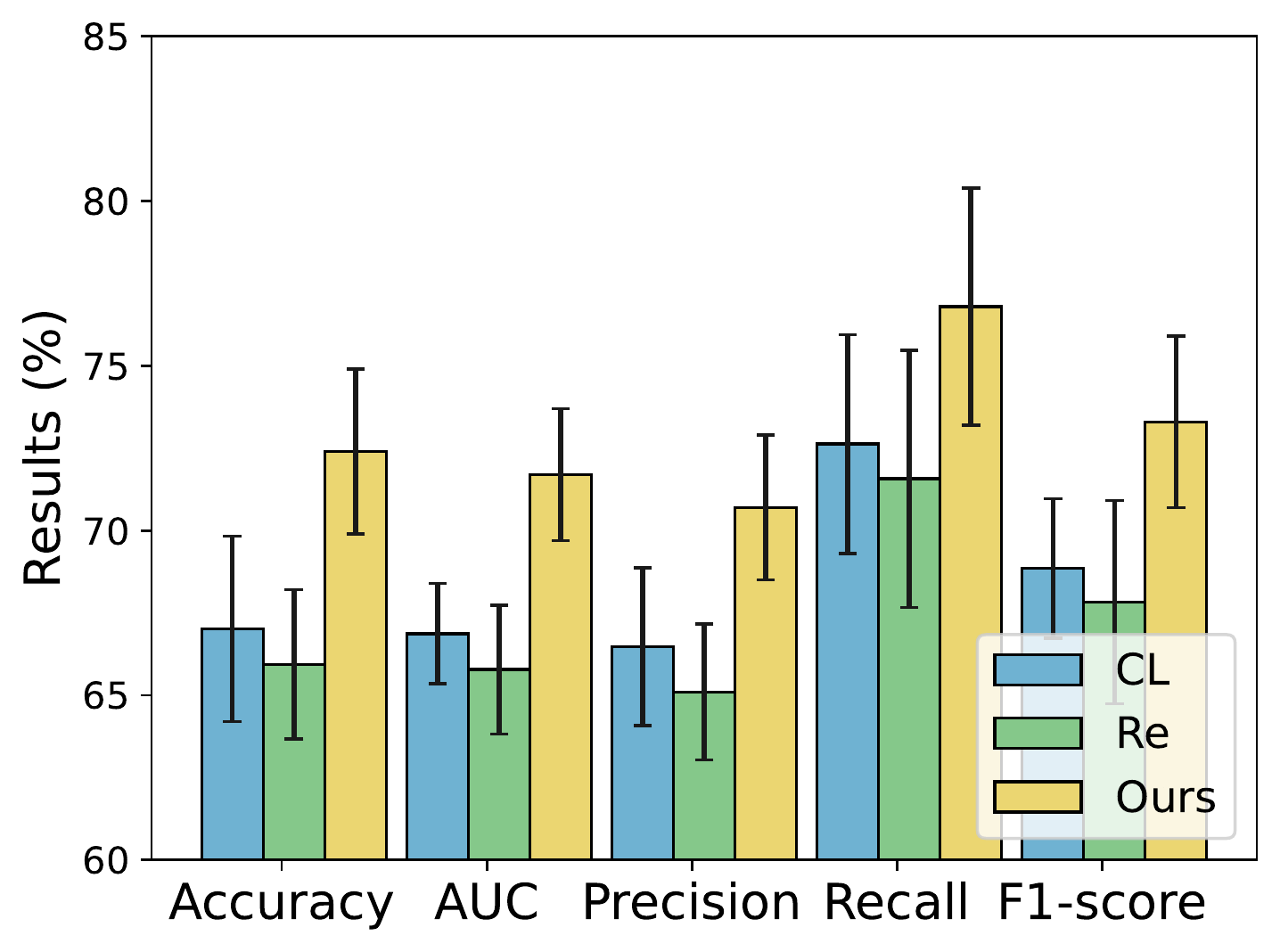}}}
\caption{Effectiveness of different SSL strategies, \ie `CL' : contrastive learning-based SSL and `Re': reconstruction-based SSL.}
\vspace{-1mm}
\label{fig_aug}
\end{figure}

\subsection{Results and analysis} 
\label{sec:results_and_analysis}
The quantitative comparison is shown in Table \ref{tab:1}.
% under the setting of unsupervised learning and 20\% labeled data fine-turning.
We can observe that: \ours{} achieves the best performance on the quantitative metrics across two datasets, followed by CCA-SSG, SAC-GCN, GAT, BGRL, Vanilla GCN, MVGRL, and DGI. 
\revisionR{Meanwhile, we find that the improvements are significant (with $p<0.05$ via significance testing) compared with state-of-the-art graph SSL methods. Particularly, \ours{} improves by $12.73\%$ ($p=4.83e^{-4}$) and $4.04\%$ ($p=4.16e^{-3}$) on average}, compared to the baseline comparison method DGI and the best graph SSL comparison method CCA-SSG. This indicates the superior performance of our proposed \ours{}. 
Moreover, we can observed that the DGI and MVGRL cannot \revisionR{outperform} the baseline vanilla GCN. The possible reason is that the limited  number  of  samples  and  a  small  number  of  classes in the dataset (the problem of class collision \cite{saunshi2019theoretical}).
We emphasize that the ultimate success of  SSL is leveraging unlabeled data, which means SSL methods can gain more benefits and achieve higher results with more unlabeled data.

We further conduct experiments under different ratios of labeled data ($10\%$ to $80\%$). Remarkably, labeled data is only used in the fine-tuning stage under the SSL setting.
As shown in Fig. \ref{fig2}, \ours{} always outperforms the vanilla GCN,  especially with a large
gap in performance on limited labeled data. 
Specifically, compared to vanilla GCN, \ours{}
achieves an average improvement of  $11.4\%$, $9.9\%$, $4.3\%$, and $2.0\%$ in terms of  $10\%$, $20\%$, $50\%$, and $80\%$ percentage of labeled data, respectively. These results confirm our claim that \ours{}  obtains satisfying prediction performance with a small portion of labeled data.

\subsection{Ablation study} 
\label{sec:ablation}

\noindent \textbf{Effectiveness of  dynamic FC augmentation.}
To evaluate the effectiveness of our proposed augmentation strategy (Sec. \ref{sec_augmentation}), we report the performance of seven augmentation strategies based on our framework in Fig. \ref{fig4} (a), \ie DA (with random drop augmentation \cite{thakoor2021bootstrapped}),
MA (with \textbf{\textit{M-A}} augmentation),
SA (with \textbf{\textit{S-A}} augmentation),
MA$+$DA (with both \textbf{\textit{M-A}} and DA),
SA$+$DA (with both \textbf{\textit{S-A}} and DA), 
SA$+$MA (with both \textbf{\textit{S-A}} and \textbf{\textit{M-A}}), 
and SA$+$MA$+$DA (with \textbf{\textit{S-A}}, \textbf{\textit{M-A}} and DA). 
Generally, we observe that performance worsens with DA only, while the proposed two augmentations (\textbf{\textit{S-A}} and \textbf{\textit{M-A}}) increase performance.
Such observations do not come as a surprise, indeed SSL relies on defining two views appropriate for specific data and tasks (\ie fMRI in this study).
Our best performance is obtained from SA$+$MA$+$DA.

\noindent \textbf{Effectiveness of different SSL strategies.}
We perform an experiment to compare different SSL strategies, \ie Contrastive-based SSL (CL), Reconstruction-based SSL (Re), and \ours{}. For CL, we replace Eq. (\ref{eq:loss}) with InfoNCE loss \cite{oord2018representation} and select negative samples randomly. For RE, we replace Eq. (\ref{eq:loss}) with MSE loss and add a decoder network.
Fig. \ref{fig_aug}  shows the superiority of \ours{} compared with other SSL strategies (CL and RE), which is consistent with our analysis. 

\begin{figure}[!t]
\centering
{
\subfloat{\scalebox{0.35}{\includegraphics{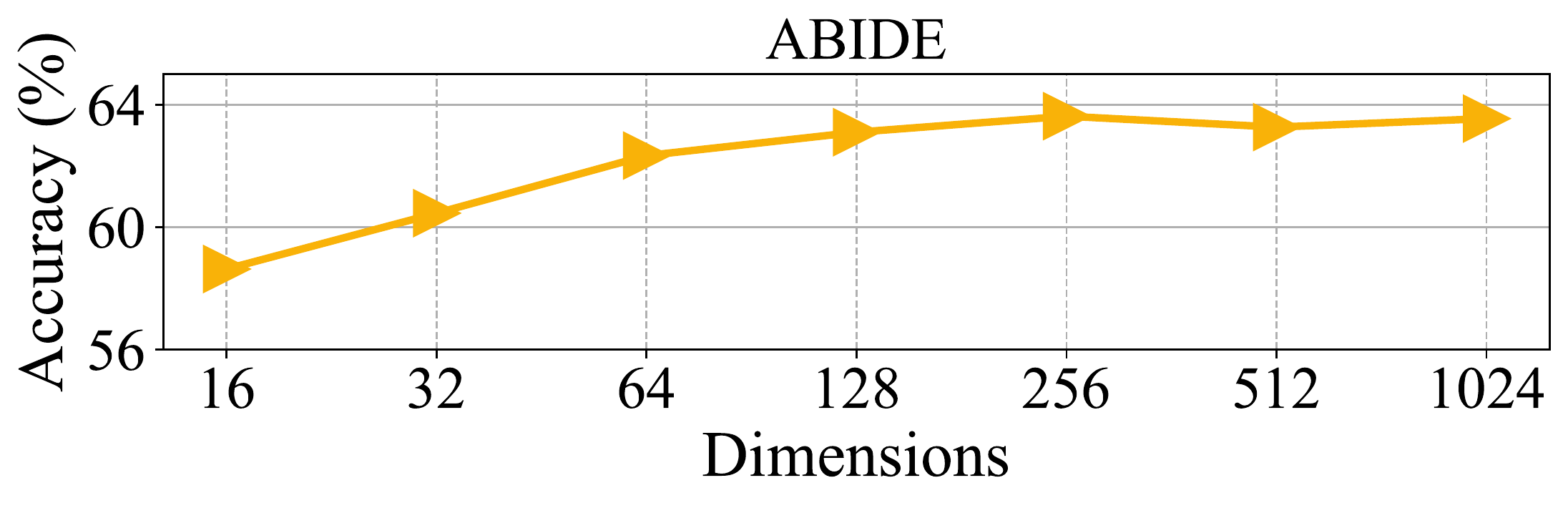}}}\\\vspace{-1mm}
\subfloat{\scalebox{0.35}{\includegraphics{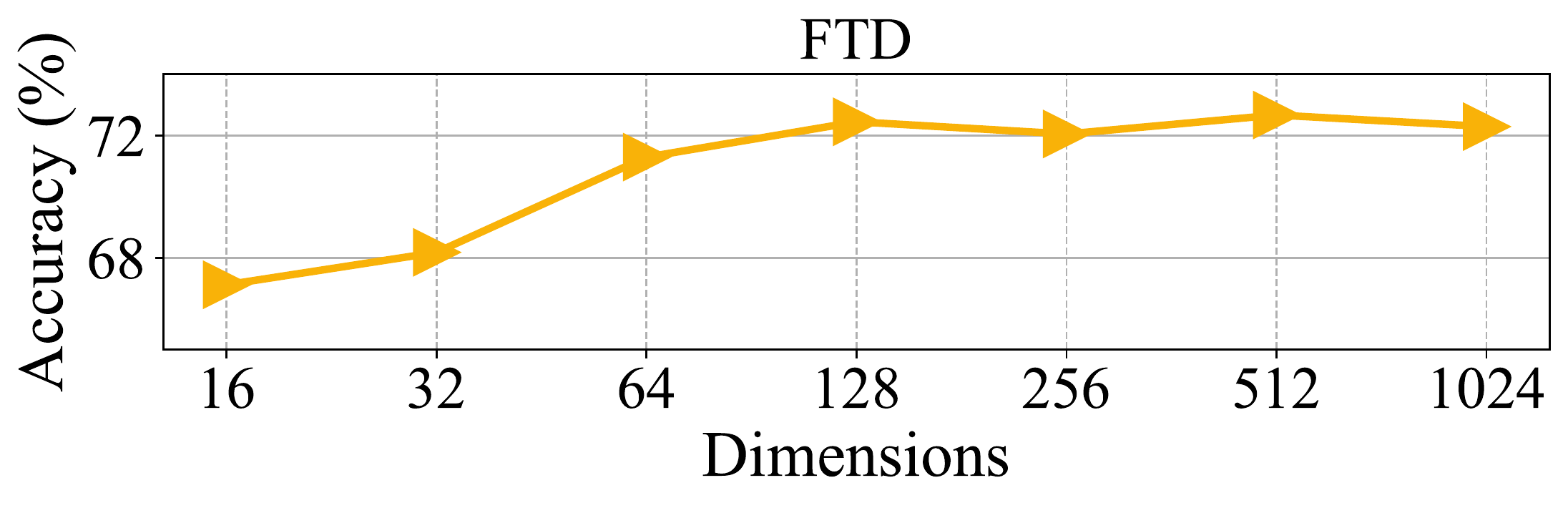}}}}
\caption{Accuracy of \ours{} with different numbers of embedding dimensions.}
\label{fig_dim}
\end{figure}

\begin{figure}[!t]
\centering
{
\subfloat{\scalebox{0.35}{\includegraphics{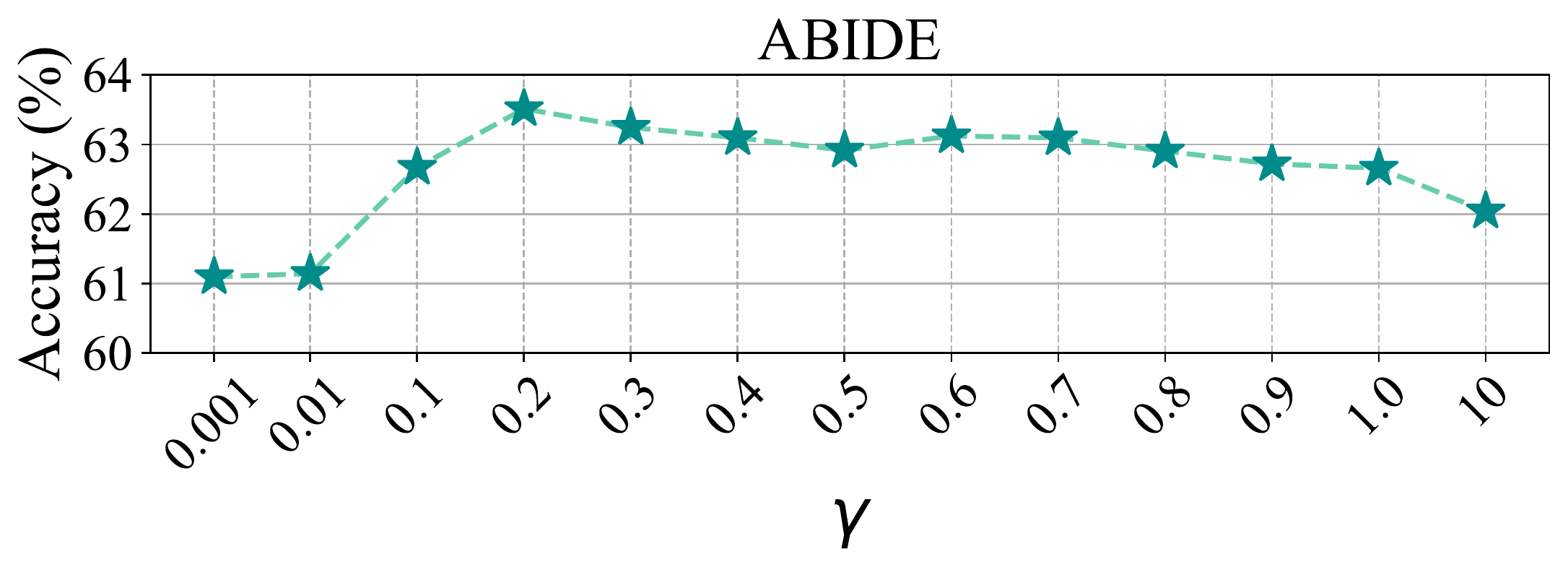}}}\\ \vspace{-3mm}
\subfloat{\scalebox{0.35}{\includegraphics{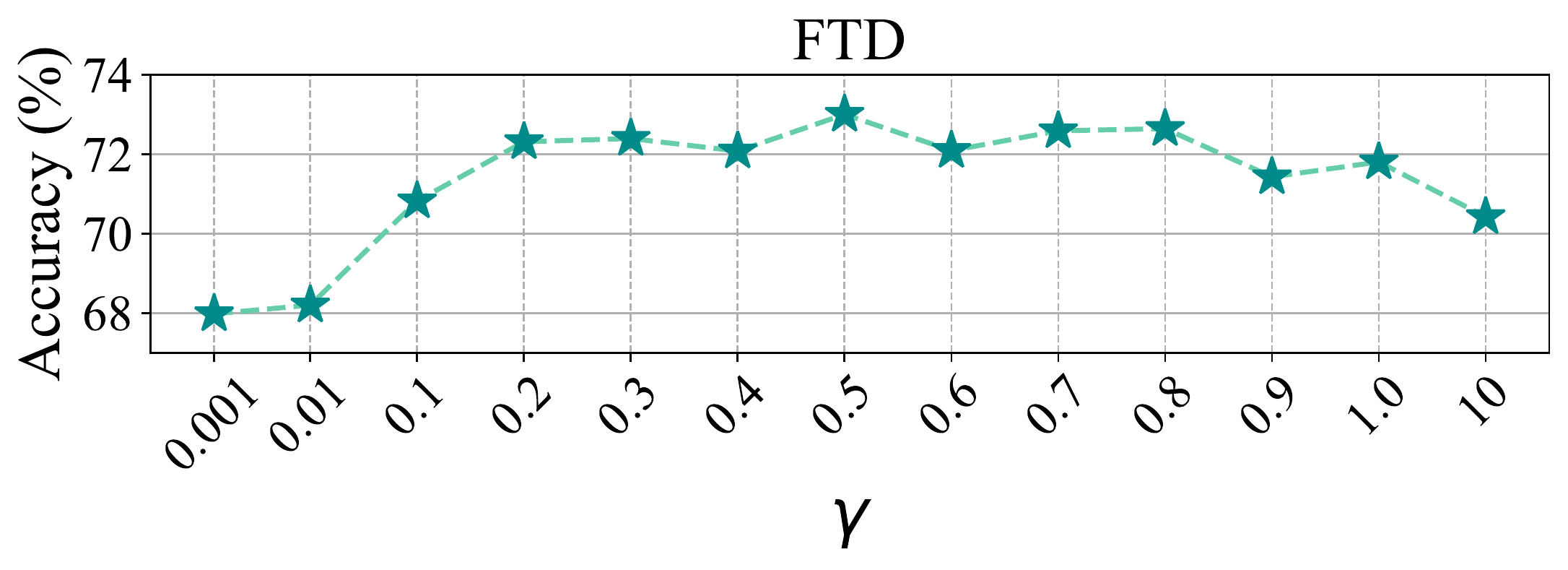}}}}
\caption{Accuracy of \ours{} with different values of $\gamma$ in Eq. (\ref{eq:loss}).}
\vspace{-2mm}
\label{fig_beta}
\end{figure}

\noindent \textbf{Different dimensional embedding.} 
We analyze the impacts of different numbers of embedding dimensions in our methods.
For this purpose, we set the range of the numbers of embedding dimensions as $\{16, 32, 64, 128, 256, 512, 1024\}$. 
As shown in Fig. \ref{fig_dim}, the accuracy of our method increased with the increasing numbers of embedding dimensions \ie from $16$ to $256$ on the ABIDE dataset and from $16$ to $128$ on the FTD dataset. On both two datasets, the performance peaks around $128$ to $256$ and then increases slowly. 
\revision{
Moreover, the  accuracy  is  slightly low  when  the  dimensionality is set in range of $[16, 64]$. The reason can be that low dimensional embeddings cannot capture the useful information well limited by its insufficient representational ability. The performance is stable for the high-dimensional embeddings, which indicate that \ours{} is insensitive to the numbers of embedding dimensions with high dimensional embeddings.  In this observation, we set the numbers of embedding dimensions to $256$ in all experiments.
}\revisionfoot{R1-Q8}

% and embeddings with a large number of embedding dimension (\eg $256$)

\begin{figure}[!t]
\centering
{
\subfloat[ABIDE]{\scalebox{0.3}{\includegraphics{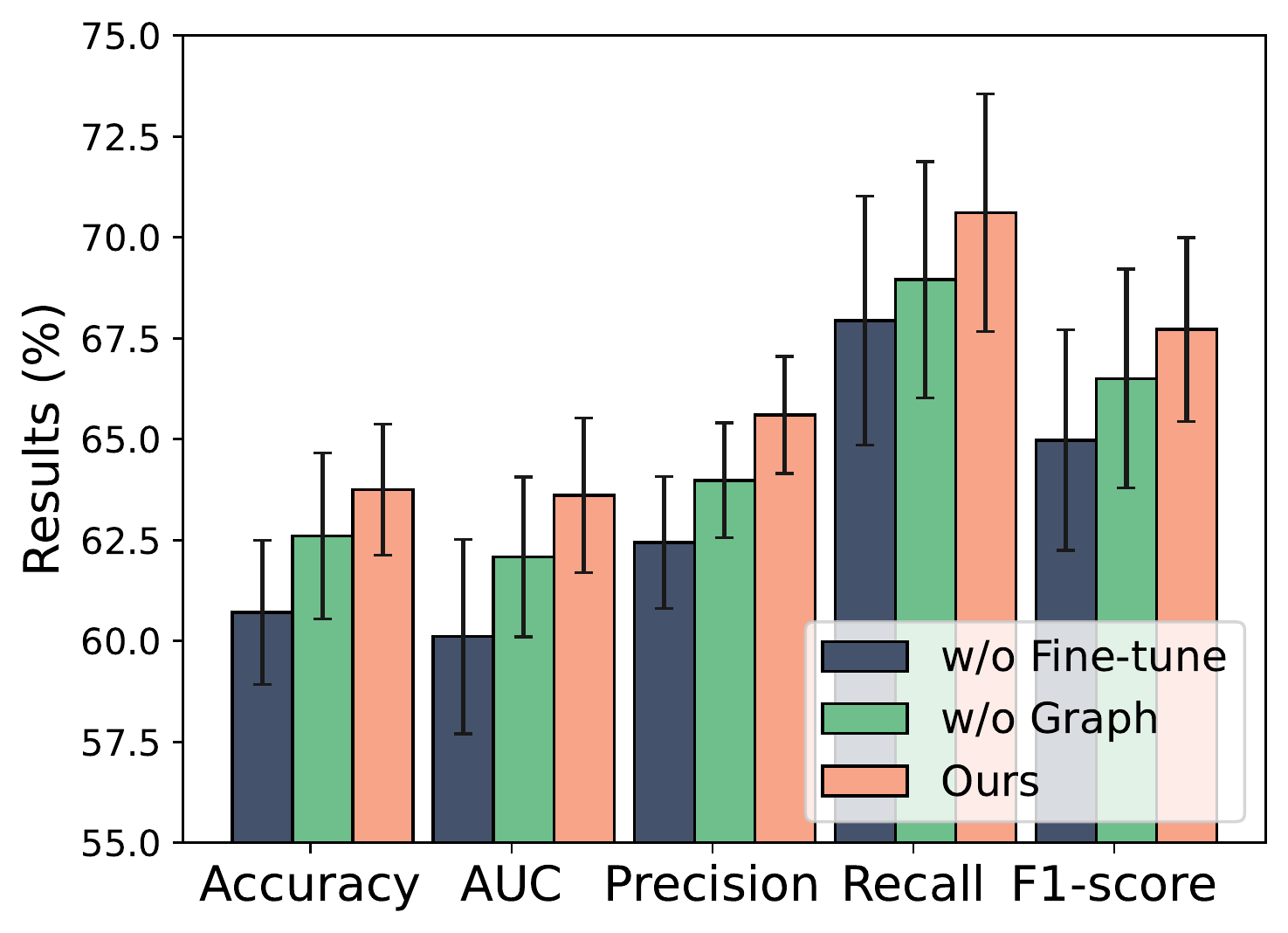}}} 
\subfloat[FTD]{\scalebox{0.3}{\includegraphics{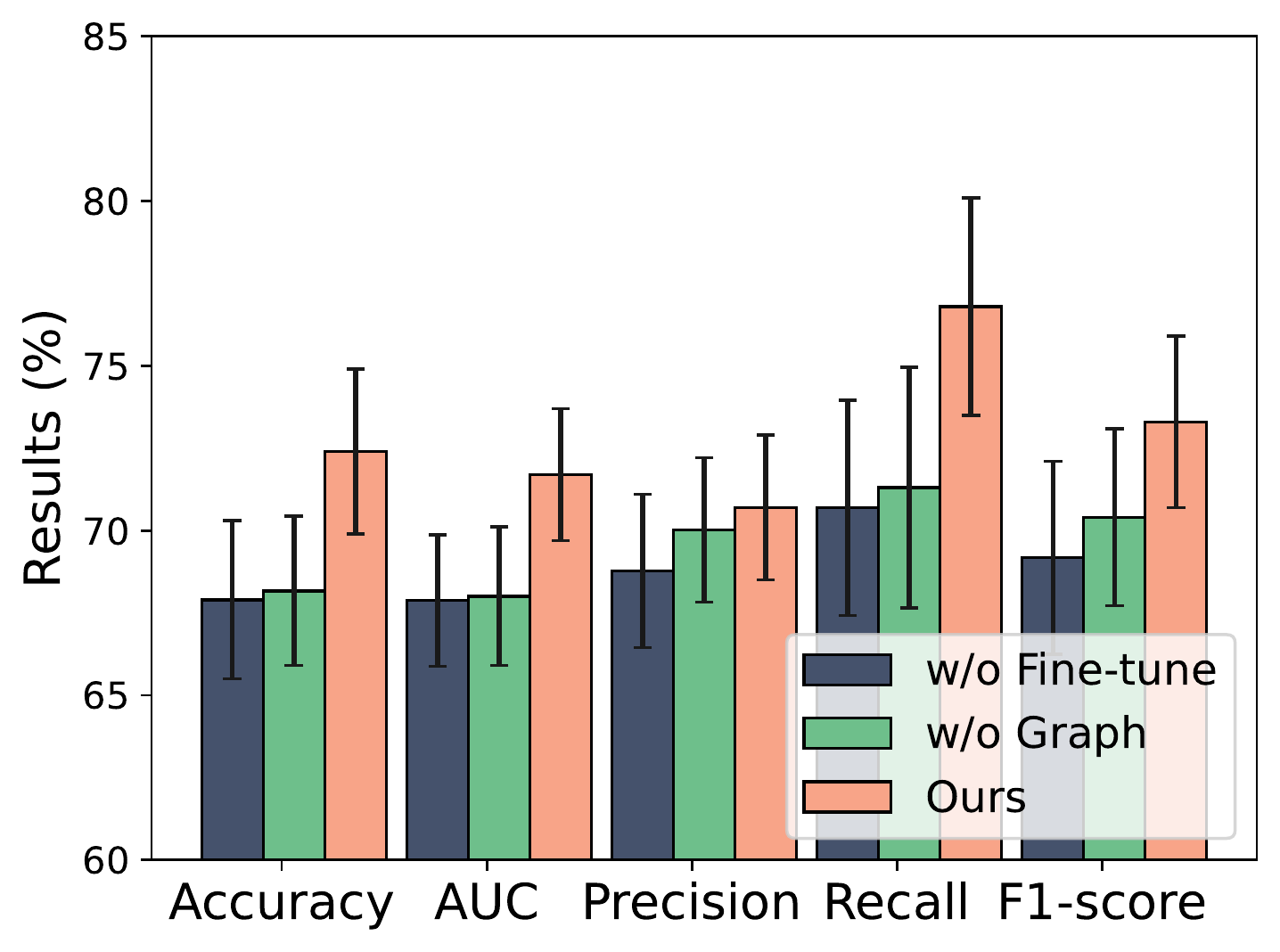}}}}
\caption{Accuracy of \ours{} without fine-tuning and \ours{} without considering graph information in SSL.}
\label{fig_fine}
\end{figure}

\noindent \textbf{Effectiveness of $\gamma$ in the objective function.}
We conduct an ablation experiment on the second term in the objective function.
The results in  Fig. \ref{fig_beta} correspond to $\gamma$ in Eq. \ref{eq:loss} varying from $0.001$ through $10$. We can see that \ours{} has poor performance on both two datasets when $\gamma$ is small (\eg $\gamma \le 0.1$ ), providing empirical evidence for our state that the second term in Eq. \ref{eq:loss} is necessary to get promising result.
Moreover, over a wide range of the value of $\gamma$ (\eg $0.1 \le \gamma \le 0.8$ ), \ours{} achieves a near-perfect performance on both two datasets, which shows \ours{} can achieve stability when $\gamma$ is in a reasonable range (\eg $0.2 \le \gamma  \le 0.8 $).
As expected, a large value of $\gamma$ in Eq. \ref{eq:loss} takes a hit on performance (\eg $0.9 \le \gamma$). The reason can be that a large value of $\gamma$ means the second term dominates the objective function, whereupon the model tends to produce more irrelevant embeddings rather than discriminative embeddings for downstream tasks.

\begin{figure}[!t]
\centering
{
\subfloat{\scalebox{0.21}{\includegraphics{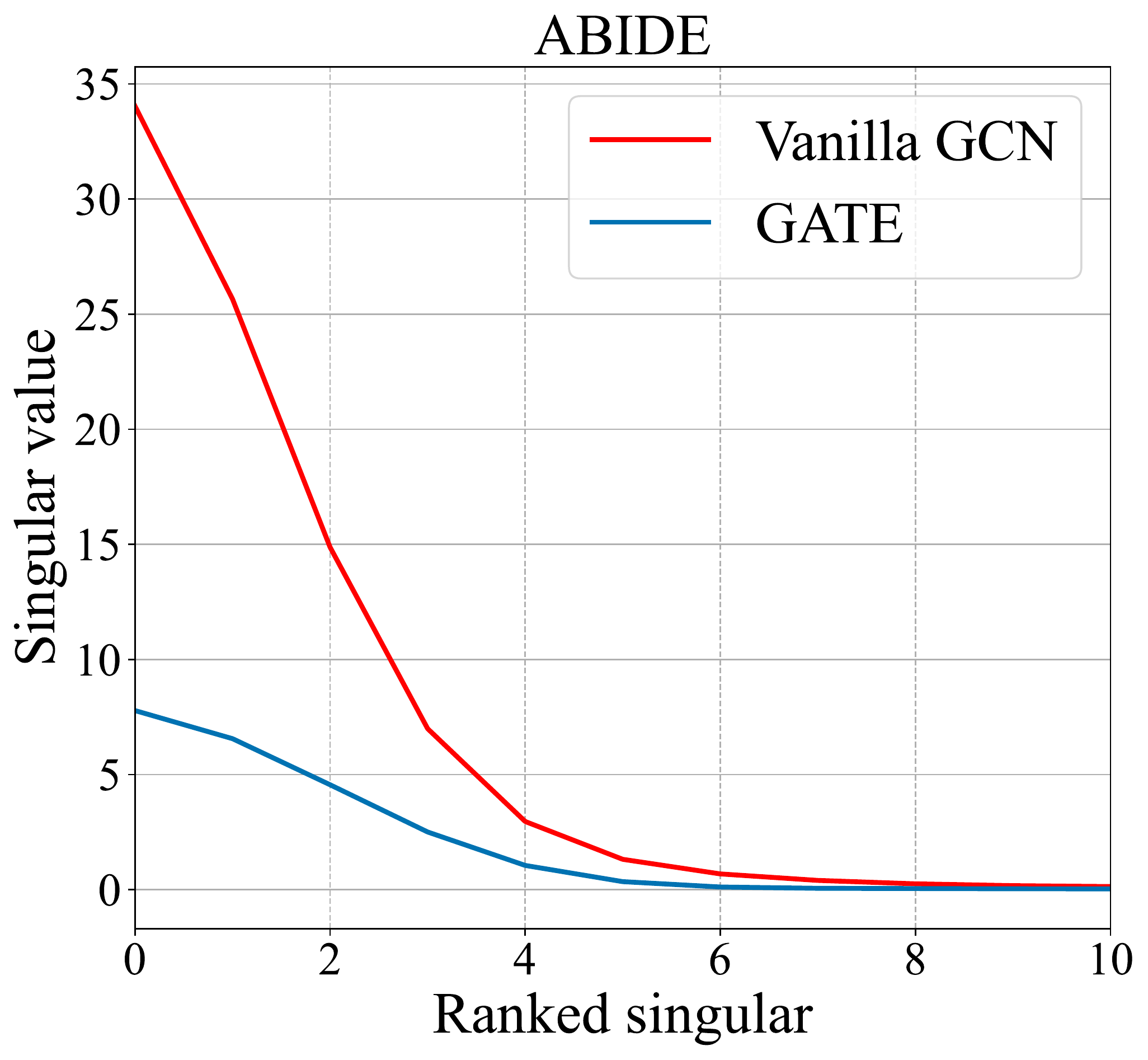}}}
\subfloat{\scalebox{0.21}{\includegraphics{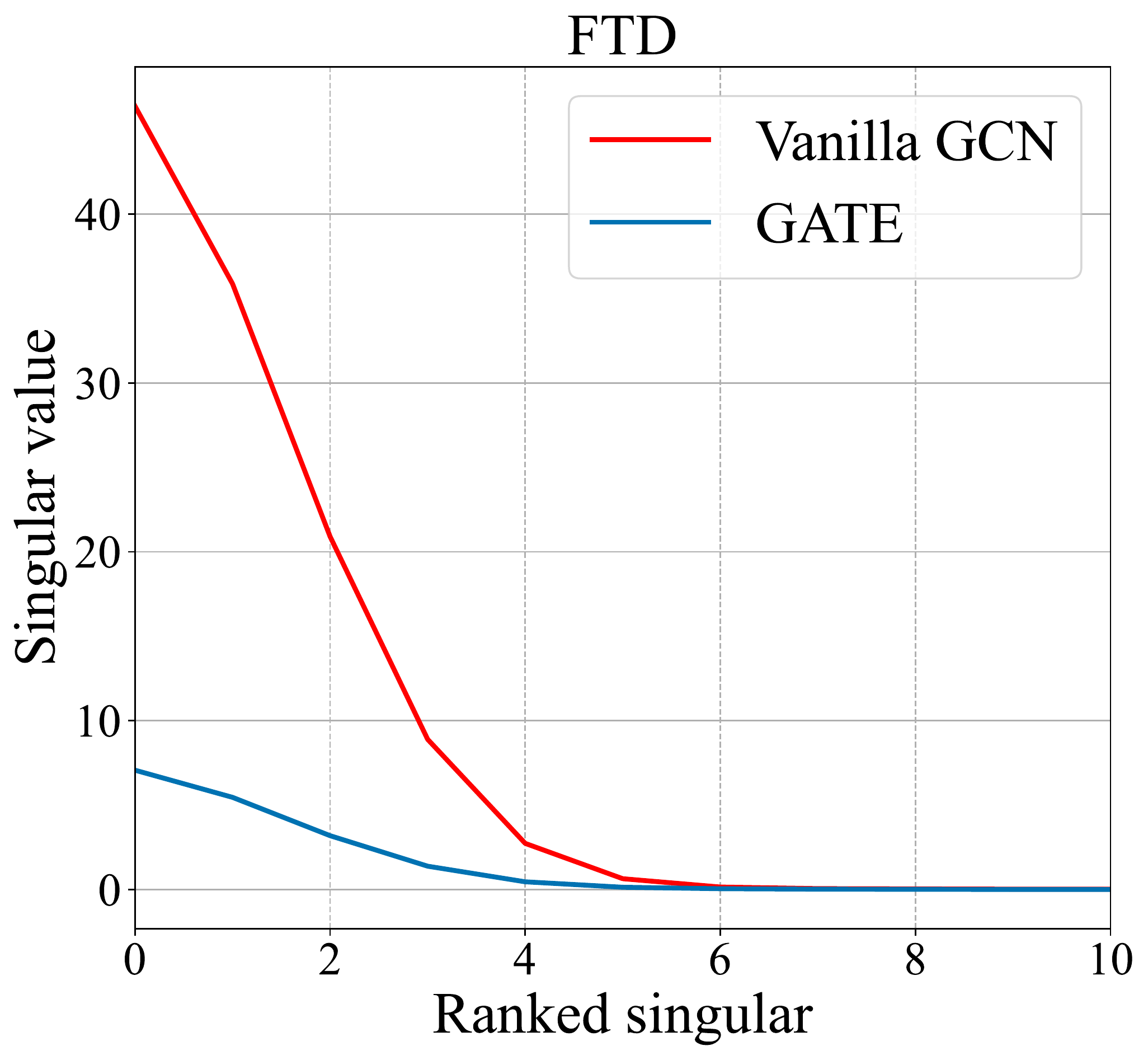}}}}
\caption{Comparison of the low rank representation between \ours{} and vanilla GCN.}
\label{fig_rank}
\end{figure}

\noindent \textbf{Effectiveness of fine-tuning and graph.}
To analyse the effectiveness of fine-tuning step and graph, we perform a comparison of \ours{} without fine-tuning step and \ours{} without graph in SSL (\ie replacing $\mathbf{A}$ with $\mathbf{I}$ in Eq. (\ref{eq:gcn}) in SSL step). 
As shown in Fig. \ref{fig_fine},  we clearly observe that both fine-tuning and graph are a necessary in our pipeline. 
In practice, fine-tuning step fits the practical medical deployment scenario where a small portion of labeled data can be obtained to fit specific disease in downstream tasks. 
Besides, graph information can provide common biomarkers between subjects for model in SSL step. Indeed, after SSL step, there is no need to keep the large graph structure for computational efficiency in fine-tuning step, which is able to make \ours{} easier to be applied in clinical practice.

\noindent \textbf{Low-rank representation.}
In Theorem \ref{th:main}, we state that \ours{} can obtain low-rank representation to decrease  upper bound of excess risk of downstream task.
We carry out Singular Value Decomposition (SVD) experiments on embeddings to support our statement.
According to Fig. \ref{fig_rank},  the singular values of \ours{} is remarkably smaller than vanilla GCN in both two datasets.

\revision{
\noindent \textbf{Parameter Sensitivity Analysis.} 
As we mentioned in Sec. \ref{sec_augmentation}, existing dynamic FC-based methods are sensitive to the parameters in the sliding window algorithm (\eg the length of the sliding window and the step of the sliding window) \cite{savva2019assessment,shakil2016evaluation}. 
We investigate  whether \ours{} is sensitive to the parameters used in the proposed two augmentation, which is related to the sliding window algorithm. 
To do this, we  conduct the parameter sensitivity studies on  proposed augmentations and report their results in Fig. \ref{fig_sa_mini} and Fig. \ref{fig_ma_mini}. 

First, in the augmentation \textbf{\textit{S-A}}, it includes two parameters, \ie the length of the sliding windows $L$ and the step of sliding windows $s$ which control the total number of the sliding windows.  Figure \ref{fig_sa_mini} shows the results obtained by traversing $L$ from $20$ to $70$ and $L$ from $10$ to $35$, with fixing the other settings. We could observe that \ours{} is insensitive to the parameters $L$ within a certain range of values and $s$.  Specifically, \ours{} can consistently 
obtain promising results except in the case of a tiny value of $L$, and the suggested value ranges of $L$ and $s$ are $[30, 60]$ and $[10, 35]$, respectively. 
Second, in the augmentation \textbf{\textit{M-A}}, we investigate the effect of gaps between two different scale sliding windows (\ie $\left \|  l_a - l_b \right \| $). 
Figure \ref{fig_ma_mini} shows the results obtained by traversing gaps from $5$ to $40$ with fixing the other settings (\eg the value of $L$ and the value of $s$ are set to $30$ and $15$, respectively). As a result,  the variation of the performance of \ours{} is relatively stable with the value ranges of gaps is  $[10, 40]$. This observation demonstrates that \ours{} is insensitive to the parameters settings of \textbf{\textit{M-A}} in our experiments.
In summary, the  reasons of the above observation could be that \ours{} maximizing the similarity of the two representations (\ie embeddings of two different sliding windows) to reduce the perturbations on different sliding windows.
} \revisionfoot{R2 R4}

\begin{figure}[h]
\centering
\scalebox{0.95}{
\subfloat[ABIDE]{\scalebox{0.4}{\includegraphics{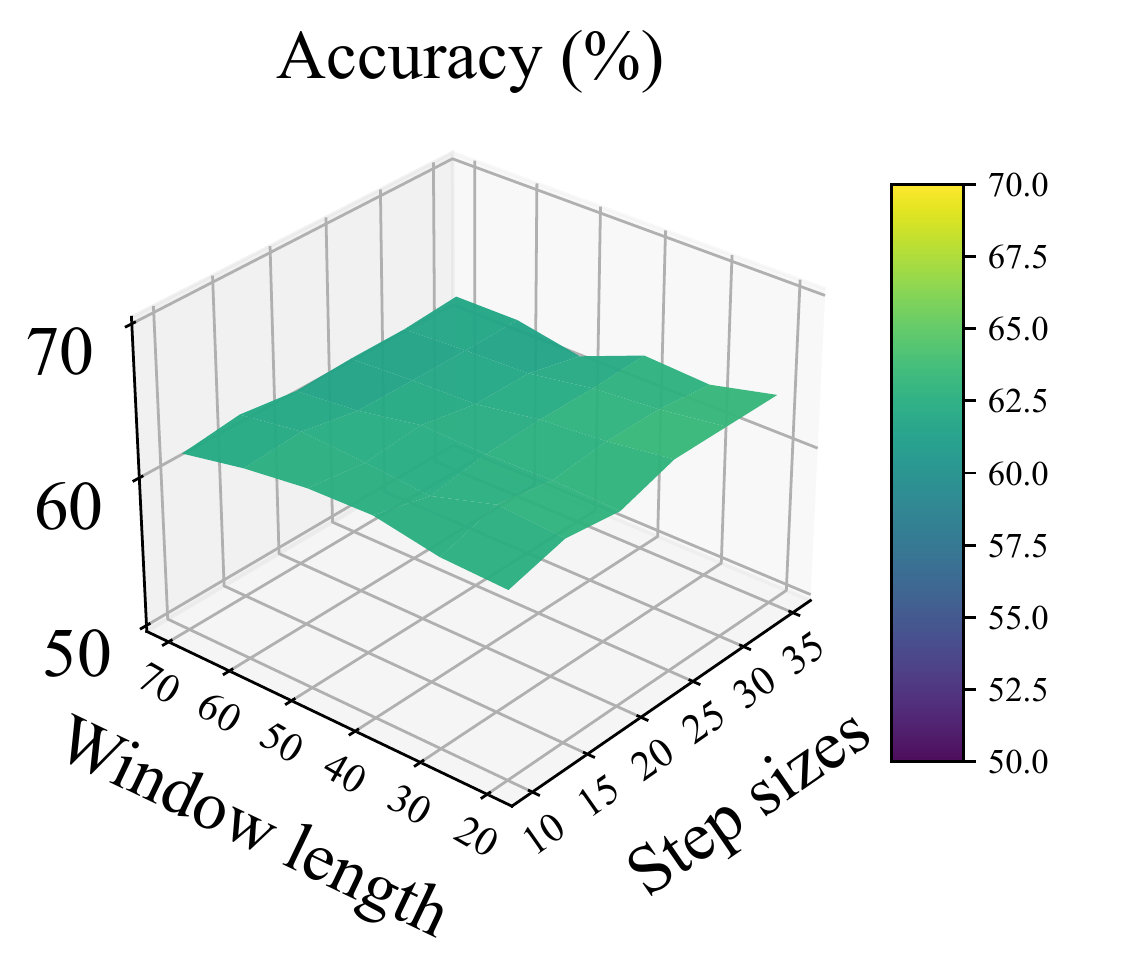}}}
\subfloat[FTD]{\scalebox{0.4}{\includegraphics{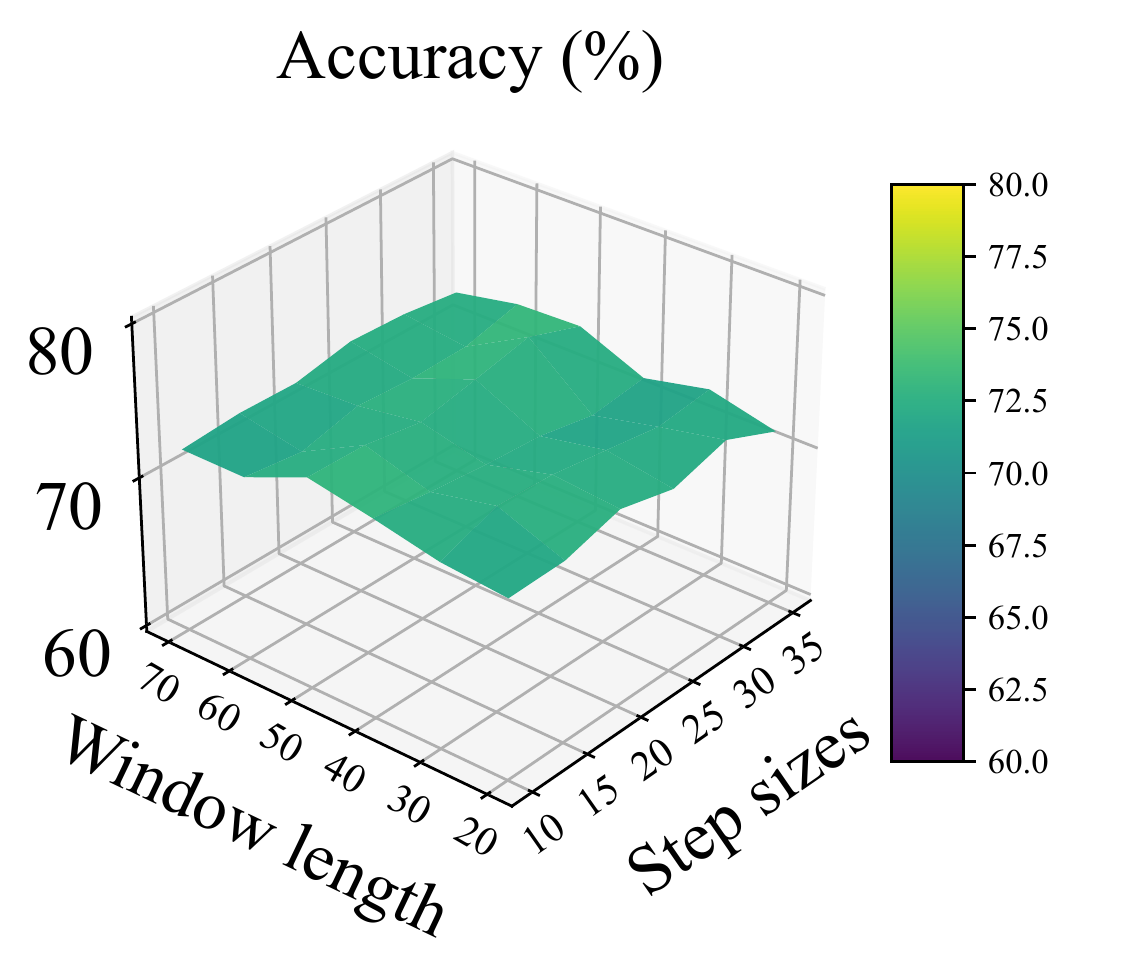}}} 
}
\caption{Accuracy of \ours{} at different parameter settings (\ie window length  and step size) on two dataset.}
\vspace{-2mm}
\label{fig_sa_mini}
\end{figure}
\begin{figure}[h]
\centering
\scalebox{1}{
\subfloat[ABIDE]{\scalebox{0.4}{\includegraphics{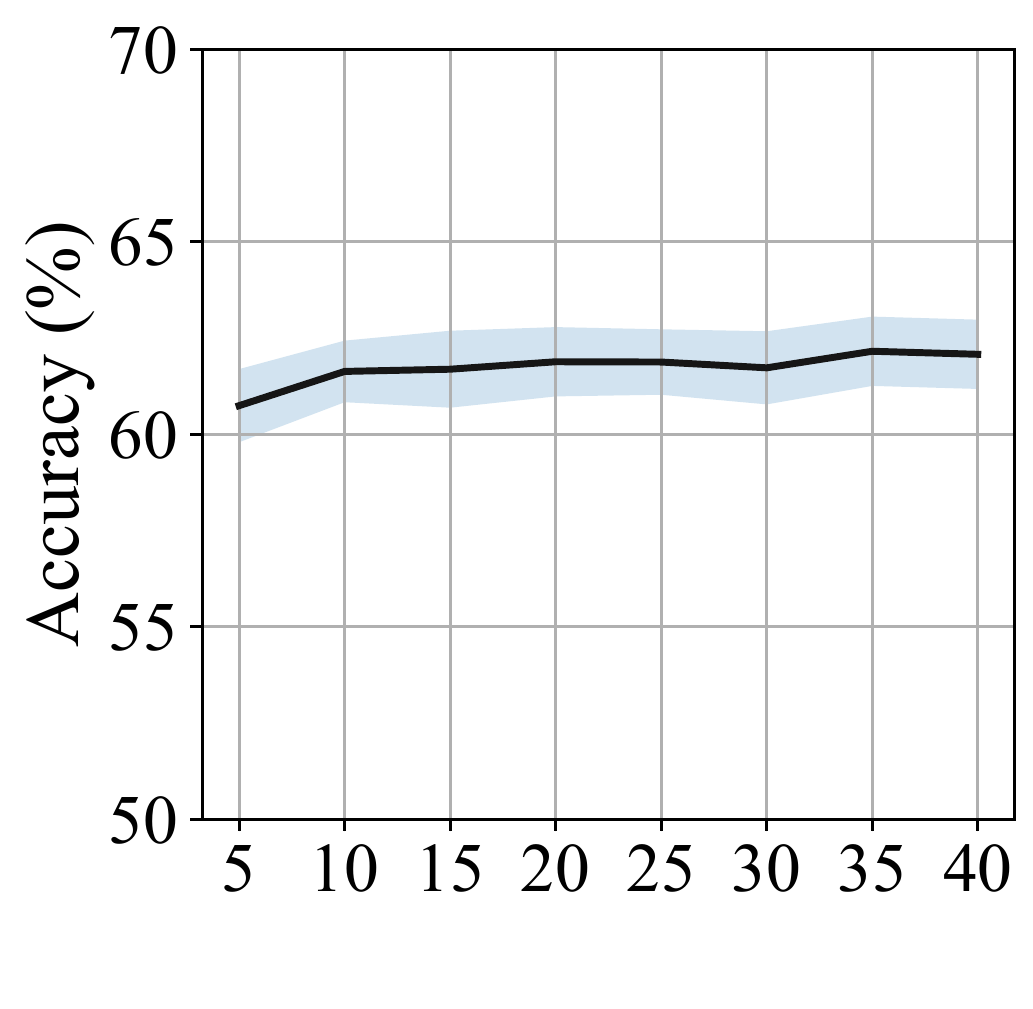}}} ~
\subfloat[FTD]{\scalebox{0.4}{\includegraphics{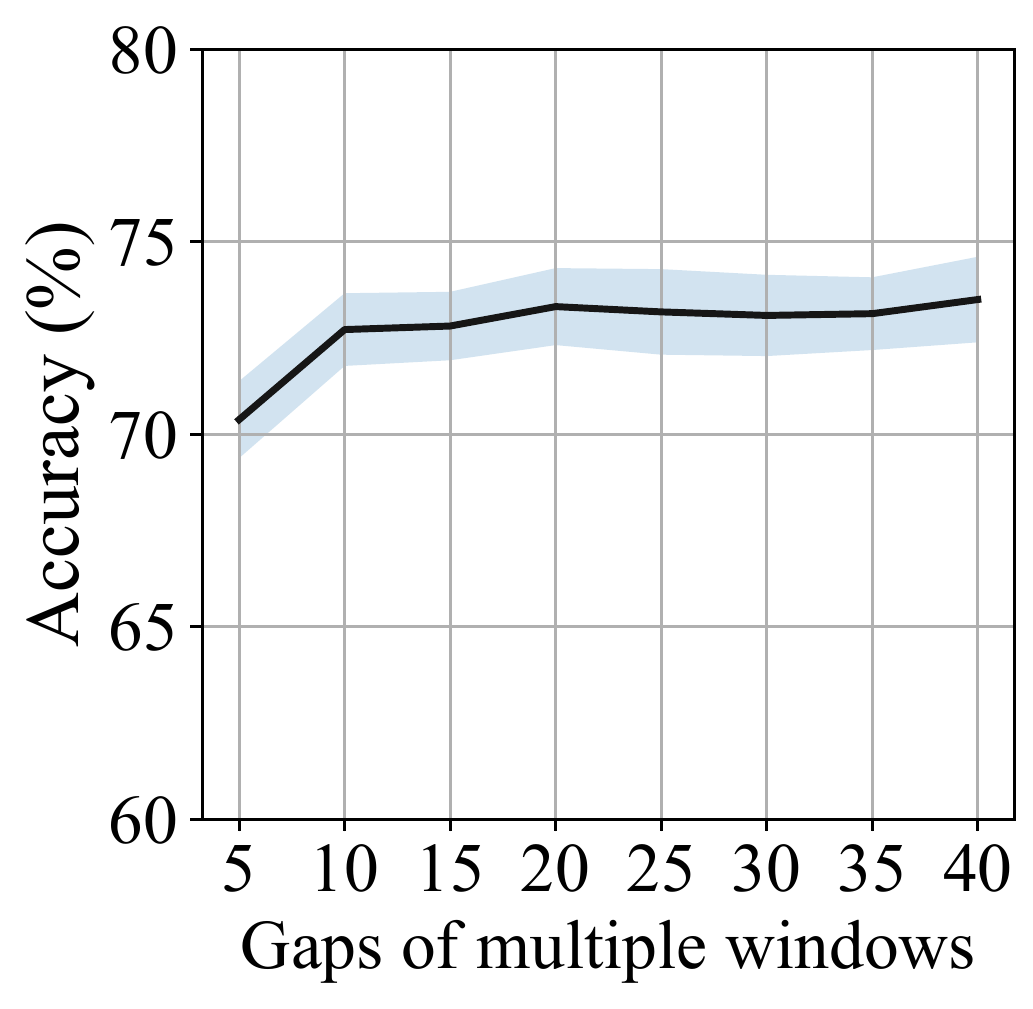}}} 
}
\caption{Accuracy of \ours{} at different parameter settings, \ie gaps of multiple windows (\ie $\|l_{a} - l_{b}\|$), on two dataset.}
\vspace{-2mm}
\label{fig_ma_mini}
\end{figure}

\section{Discussion} \label{sec_discussion}
\subsection{The needs of label efficiency for fMRI}
In recent years, fMRI analysis has revolutionized with deep learning, while training a powerful model usually requires a large amount of labeled fMRI data.
To visually explain the needs of large labeled data for fMRI analysis, we depict the results of vanilla GCN under different label rates in Fig.~\ref{fig2}. We can observe that the accuracy of the vanilla GCN decreases significantly with the decreasing labeling rate (\eg 10\% to 50\%).
\revisionR{However obtain a sufficient number of labeled images for fMRI data could be challenging due to the cost of obtaining neuro-disease annotations.}
Therefore, designing a method for limited labeled fMRI data is crucial~\cite{wang2019identifying}. %\Li{Nan, add citation.}
To solve this limitation, we \revisionR{propose} a graph SSL method \ours{}, which can yield comparable results under label-efficient data (\eg the accuracy of \ours{} under 20\% labeling rate is comparable to the accuracy of vanilla GCN under 50\% labeling rate).
% It can be seen from Fig.~\ref{fig2}, with a very low labeling rates (\eg 10\% to 20\%), \ours{} can yield comparable results. 
% Specifically, we can see that in Fig.~\ref{fig2}, which represents the model trained with an extensive variation range, consisting from 0.001 to 10. 
% Compared to the vanilla GCN, \ours{} can obtain a  significant improvement under label-efficient setting (with 10\% and 20\% of labeled fMRI data).
% With 10\% and 20\% of label rates, we can obtain relatively better results on ABIDE and FID dataset, respectively.
The main reason could be that \ours{} can help extract robust features on label-efficient data.
Consequently, our experiments show promising results to tackle the issue of limited labeled fMRI data.

% When the noise rate is low, such as 0.01 and 0.01, the performance of the model is still comparable. When performing data augmentation, for example, the label rate is expanded to 0.2, \ours{} using it produces the overall better performance, compared with that using other label rates. 

% Accordingly, suggesting varying the  may affect the model performance.

% The results of different label rates, the main reason could be that \ours{} can extract abundant features on label-efficient data by SSL.

\subsection{Graph learning for neuroimaging}

Deep graph learning models, like GCNs, have become popular approaches for populations disease analysis, including neuroimaging \cite{parisot2018disease,li2018brain}. Modeling subject-wise neuroimaging signals (and auxiliary phenotype features) could leverage the representation of similar subjects.  Our work investigates how adding similarity regularization affects GCN-based representation learning on fMRI signals. 
As shown in Fig. \ref{fig_fine} and Tab.~\ref{tab:1}, graph information is essential to get good results and \ours{} is significantly better than other GCN-based methods.
The reason could be better extracting associations between subjects and maximizing the correlation of the representations from multiple coupled views. Hence, we can conclude that graph learning can obtain embedding matrices of different views to characterize the representations of neuroimaging. 

% \revision{
% \subsection{Graph learning encoder} \label{sec_graph_encoder}
% To do this, we conduct the experiments with different graph learning encoder and report their results in Table \ref{}.

% Although the performance of our method cloud be further improved with more powerful graph learning encoder, 

% \input{z_Table4}
% }

\subsection{Technical contributions of our work} \label{sec_Technical}

We further analyze the technical contributions of \ours{}. 
% Firstly, \ours{} utilizes an SSL-based strategy to capture  vital relationship for generation multiple coupled views. 
\revisionR{Firstly, \ours{} utilizes an SSL-based strategy to capture vital relationship for generation multiple coupled views of a fMRI BOLD signal.}
We can achieve good performance of pre-training on unlabeled fMRI data and fine-tuning on small labeled
data. 
Then, we design a GCN encoder for limited labeled fMRI data to extract associations between subjects. Our GCN-based method effectively solves the issue of earning on label-efficient dataset.
\revision{
Moreover, \ours{} is a versatile method which can be further improved by considering a more powerful GNN encoder (\ie $f (\cdot)$) and combining a more powerful classifier (\ie $\psi(\cdot)$) for real application.
}
In general, self-supervised tasks provide informative priors that can benefit GCN in generalizable target performance.

\section{Conclusion}
In this paper, we propose \ours{}, a novel graph SSL method for fMRI analysis on the population graph. \ours{} can achieve outstanding classification performance without labeled data in the training phase.
Theoretical analysis analyzes demonstrate the appropriateness of the proposed augmentation strategy, CCA regularization loss in SSL for fMRI analysis, and the generalization bound of our method.
Experimental results show that \ours{} outperforms alternative SSL methods on two neuroimaging datasets with significant margins. 
Furthermore, it opens perspectives to bring deep learning from research to clinic by learning from noisy and limited labeled data.

\section{Acknowledgment}

This work was partially supported by the National Natural Science Foundation of
China (Grant No. 61876046), 
Medico-Engineering Cooperation Funds from University of Electronic Science and Technology of China (No. ZYGX2022YGRH009
and ZYGX2022YGRH014), 
the Guangxi “Bagui” Teams for Innovation and
Research, China, 
Natural Sciences and Engineering Research Council of Canada (GECR-2022-00430), 
NVIDIA Hardware Award, 
and Public Safety Canada (NS-5001-22170).

\normalem{

}


\begin{thebibliography}{10}
\providecommand{\url}[1]{#1}
\csname url@samestyle\endcsname
\providecommand{\newblock}{\relax}
\providecommand{\bibinfo}[2]{#2}
\providecommand{\BIBentrySTDinterwordspacing}{\spaceskip=0pt\relax}
\providecommand{\BIBentryALTinterwordstretchfactor}{4}
\providecommand{\BIBentryALTinterwordspacing}{\spaceskip=\fontdimen2\font plus
\BIBentryALTinterwordstretchfactor\fontdimen3\font minus
  \fontdimen4\font\relax}
\providecommand{\BIBforeignlanguage}[2]{{%
\expandafter\ifx\csname l@#1\endcsname\relax
\typeout{** WARNING: IEEEtran.bst: No hyphenation pattern has been}%
\typeout{** loaded for the language `#1'. Using the pattern for}%
\typeout{** the default language instead.}%
\else
\language=\csname l@#1\endcsname
\fi
#2}}
\providecommand{\BIBdecl}{\relax}
\BIBdecl

\bibitem{10.1227NEU}
S.~Lang, N.~Duncan, and G.~Northoff, ``{Resting-State Functional Magnetic
  Resonance Imaging: Review of Neurosurgical Applications},''
  \emph{Neurosurgery}, vol.~74, no.~5, pp. 453--465, 2014.

\bibitem{DynamicFC}
S.~Menon and K.~Krishnamurthy, ``A comparison of static and dynamic functional
  connectivities for identifying subjects and biological sex using intrinsic
  individual brain connectivity,'' \emph{Scientific Reports}, vol.~9, 2019.

\bibitem{li2018brain}
X.~Li, Y.~Zhou, N.~Dvornek, M.~Zhang, S.~Gao, J.~Zhuang, D.~Scheinost, L.~H.
  Staib, P.~Ventola, and J.~S. Duncan, ``Braingnn: Interpretable brain graph
  neural network for fmri analysis,'' vol.~74, 2021, p. 102233.

\bibitem{dvornek2019jointly}
N.~C. Dvornek, X.~Li, J.~Zhuang, and J.~S. Duncan, ``Jointly discriminative and
  generative recurrent neural networks for learning from fmri,'' in
  \emph{MLMI}.\hskip 1em plus 0.5em minus 0.4em\relax Springer, 2019, pp.
  382--390.

\bibitem{parisot2018disease}
S.~Parisot, S.~I. Ktena, E.~Ferrante, M.~Lee, R.~Guerrero, B.~Glocker, and
  D.~Rueckert, ``Disease prediction using graph convolutional networks:
  application to autism spectrum disorder and alzheimer’s disease,''
  \emph{Medical Image Analysis}, vol.~48, pp. 117--130, 2018.

\bibitem{zhou2021contrast}
Y.~Zhou, T.~Zhou, T.~Zhou, H.~Fu, J.~Liu, and L.~Shao, ``Contrast-attentive
  thoracic disease recognition with dual-weighting graph reasoning,''
  \emph{IEEE Transactions on Medical Imaging}, vol.~40, no.~4, pp. 1196--1206,
  2021.

\bibitem{ghorbani2022ra}
M.~Ghorbani, A.~Kazi, M.~S. Baghshah, H.~R. Rabiee, and N.~Navab, ``Ra-gcn:
  Graph convolutional network for disease prediction problems with imbalanced
  data,'' \emph{Medical Image Analysis}, vol.~75, p. 102272, 2022.

\bibitem{bessadok2021graph}
A.~Bessadok, M.~A. Mahjoub, and I.~Rekik, ``Graph neural networks in network
  neuroscience,'' \emph{arXiv preprint arXiv:2106.03535}, 2021.

\bibitem{hyman2020executive}
S.~L. Hyman, S.~E. Levy, S.~M. Myers, D.~Kuo, S.~Apkon, T.~Brei, L.~F.
  Davidson, B.~E. Davis, K.~A. Ellerbeck, G.~H. Noritz \emph{et~al.},
  ``Executive summary: identification, evaluation, and management of children
  with autism spectrum disorder,'' \emph{Pediatrics}, vol. 145, no.~1, 2020.

\bibitem{luo2017label}
Z.~Luo, Y.~Zou, J.~Hoffman, and L.~F. Fei-Fei, ``Label efficient learning of
  transferable representations acrosss domains and tasks,'' in \emph{NIPS},
  2017.

\bibitem{sun2021context}
L.~Sun, K.~Yu, and K.~Batmanghelich, ``Context matters: Graph-based
  self-supervised representation learning for medical images,'' in \emph{AAAI},
  vol.~35, 2021, pp. 4874--4882.

\bibitem{chen2020simple}
T.~Chen, S.~Kornblith, M.~Norouzi, and G.~Hinton, ``A simple framework for
  contrastive learning of visual representations,'' in \emph{ICML}, 2020, pp.
  1597--1607.

\bibitem{NEURIPS2020_63c3ddcc}
C.-Y. Chuang, J.~Robinson, Y.-C. Lin, A.~Torralba, and S.~Jegelka, ``Debiased
  contrastive learning,'' in \emph{NIPS}, vol.~33, 2020, pp. 8765--8775.

\bibitem{oord2018representation}
A.~v.~d. Oord, Y.~Li, and O.~Vinyals, ``Representation learning with
  contrastive predictive coding,'' \emph{arXiv preprint arXiv:1807.03748},
  2018.

\bibitem{mo2022simple}
Y.~Mo, L.~Peng, J.~Xu, X.~Shi, and X.~Zhu, ``Simple unsupervised graph
  representation learning,'' in \emph{AAAI}, 2022.

\bibitem{he2021masked}
K.~He, X.~Chen, S.~Xie, Y.~Li, P.~Doll{\'a}r, and R.~Girshick, ``Masked
  autoencoders are scalable vision learners,'' \emph{arXiv preprint
  arXiv:2111.06377}, 2021.

\bibitem{qiu2022vgaer}
C.~Qiu, Z.~Huang, W.~Xu, and H.~Li, ``Vgaer: graph neural network
  reconstruction based community detection,'' in \emph{AAAI}, 2022.

\bibitem{grill2020bootstrap}
J.-B. Grill, F.~Strub, F.~Altch{\'e}, C.~Tallec, P.~Richemond, E.~Buchatskaya,
  C.~Doersch, B.~Avila~Pires, Z.~Guo, M.~Gheshlaghi~Azar \emph{et~al.},
  ``Bootstrap your own latent-a new approach to self-supervised learning,'' in
  \emph{NIPS}, vol.~33, 2020, pp. 21\,271--21\,284.

\bibitem{he2020momentum}
K.~He, H.~Fan, Y.~Wu, S.~Xie, and R.~Girshick, ``Momentum contrast for
  unsupervised visual representation learning,'' in \emph{CVPR}, 2020, pp.
  9729--9738.

\bibitem{ni2021close}
R.~Ni, M.~Shu, H.~Souri, M.~Goldblum, and T.~Goldstein, ``The close
  relationship between contrastive learning and meta-learning,'' in
  \emph{ICLR}, 2021.

\bibitem{tian2020makes}
Y.~Tian, C.~Sun, B.~Poole, D.~Krishnan, C.~Schmid, and P.~Isola, ``What makes
  for good views for contrastive learning?'' \emph{Advances in Neural
  Information Processing Systems}, vol.~33, pp. 6827--6839, 2020.

\bibitem{wang2022chaos}
Y.~Wang, Q.~Zhang, Y.~Wang, J.~Yang, and Z.~Lin, ``Chaos is a ladder: A new
  theoretical understanding of contrastive learning via augmentation overlap,''
  in \emph{ICLR}, 2022.

\bibitem{chen2020self}
Y.~Chen, C.~Wei, A.~Kumar, and T.~Ma, ``Self-training avoids using spurious
  features under domain shift,'' in \emph{NIPS}, 2020, pp. 21\,061--21\,071.

\bibitem{thompson1984canonical}
B.~Thompson, \emph{Canonical correlation analysis: Uses and
  interpretation}.\hskip 1em plus 0.5em minus 0.4em\relax Sage, 1984, no.~47.

\bibitem{friman2003adaptive}
O.~Friman, M.~Borga, P.~Lundberg, and H.~Knutsson, ``Adaptive analysis of fmri
  data,'' \emph{NeuroImage}, vol.~19, no.~3, pp. 837--845, 2003.

\bibitem{kam2018novel}
T.-E. Kam, H.~Zhang, and D.~Shen, ``A novel deep learning framework on brain
  functional networks for early mci diagnosis,'' in \emph{MICCAI}, 2018, pp.
  293--301.

\bibitem{yao2021mutual}
D.~Yao, J.~Sui, M.~Wang, E.~Yang, Y.~Jiaerken, N.~Luo \emph{et~al.}, ``A mutual
  multi-scale triplet graph convolutional network for classification of brain
  disorders using functional or structural connectivity,'' \emph{IEEE
  Transactions on Medical Imaging}, vol.~40, no.~4, pp. 1279--1289, 2021.

\bibitem{wang2022multi}
N.~Wang, D.~Yao, L.~Ma, and M.~Liu, ``Multi-site clustering and nested feature
  extraction for identifying autism spectrum disorder with resting-state
  fmri,'' \emph{Medical Image Analysis}, vol.~75, p. 102279, 2022.

\bibitem{heeger2002does}
D.~J. Heeger and D.~Ress, ``What does fmri tell us about neuronal activity?''
  \emph{Nature reviews neuroscience}, vol.~3, no.~2, pp. 142--151, 2002.

\bibitem{huang2020attention}
J.~Huang, L.~Zhou, L.~Wang, and D.~Zhang, ``Attention-diffusion-bilinear neural
  network for brain network analysis,'' \emph{IEEE Transactions on Medical
  Imaging}, vol.~39, no.~7, pp. 2541--2552, 2020.

\bibitem{zhang2016neural}
J.~Zhang, W.~Cheng, Z.~Liu, K.~Zhang, X.~Lei, Y.~Yao \emph{et~al.}, ``Neural,
  electrophysiological and anatomical basis of brain-network variability and
  its characteristic changes in mental disorders,'' \emph{Brain}, vol. 139,
  no.~8, pp. 2307--2321, 2016.

\bibitem{mao2017gender}
N.~Mao, H.~Zheng, Z.~Long, L.~Yao, and X.~Wu, ``Gender differences in dynamic
  functional connectivity based on resting-state fmri,'' in \emph{EMBC}, 2017,
  pp. 2940--2943.

\bibitem{mertzios2019sliding}
G.~B. Mertzios, H.~Molter, and V.~Zamaraev, ``Sliding window temporal graph
  coloring,'' in \emph{AAAI}, vol.~33, no.~01, 2019, pp. 7667--7674.

\bibitem{wang2019spatial}
M.~Wang, C.~Lian, D.~Yao, D.~Zhang, M.~Liu, and D.~Shen, ``Spatial-temporal
  dependency modeling and network hub detection for functional mri analysis via
  convolutional-recurrent network,'' \emph{IEEE Transactions on Biomedical
  Engineering}, vol.~67, no.~8, pp. 2241--2252, 2019.

\bibitem{yao2020temporal}
D.~Yao, J.~Sui, E.~Yang, P.-T. Yap, D.~Shen, and M.~Liu, ``Temporal-adaptive
  graph convolutional network for automated identification of major depressive
  disorder using resting-state fmri,'' in \emph{MLMI}, 2020.

\bibitem{yu2022disentangling}
X.~Yu, L.~Zhang, L.~Zhao, Y.~Lyu, T.~Liu, and D.~Zhu, ``Disentangling
  spatial-temporal functional brain networks via twin-transformers,''
  \emph{arXiv preprint arXiv:2204.09225}, 2022.

\bibitem{shen2017deep}
D.~Shen, G.~Wu, and H.-I. Suk, ``Deep learning in medical image analysis,''
  \emph{Annual review of biomedical engineering}, vol.~19, pp. 221--248, 2017.

\bibitem{fan2019birnet}
J.~Fan, X.~Cao, P.-T. Yap, and D.~Shen, ``Birnet: Brain image registration
  using dual-supervised fully convolutional networks,'' \emph{Medical Image
  Analysis}, vol.~54, pp. 193--206, 2019.

\bibitem{kazi2019inceptiongcn}
A.~Kazi, S.~Shekarforoush, S.~A. Krishna, H.~Burwinkel, G.~Vivar,
  K.~Kort{\"u}m, S.-A. Ahmadi, S.~Albarqouni, and N.~Navab, ``Inceptiongcn:
  receptive field aware graph convolutional network for disease prediction,''
  in \emph{IPMI}.\hskip 1em plus 0.5em minus 0.4em\relax Springer, 2019, pp.
  73--85.

\bibitem{xing2019dynamic}
X.~Xing, Q.~Li, H.~Wei, M.~Zhang, Y.~Zhan, X.~S. Zhou, Z.~Xue, and F.~Shi,
  ``Dynamic spectral graph convolution networks with assistant task training
  for early mci diagnosis,'' in \emph{MICCAI}, 2019, pp. 639--646.

\bibitem{DGI}
P.~Velickovic, W.~Fedus, W.~L. Hamilton, P.~Li{\`o}, Y.~Bengio, and R.~D.
  Hjelm, ``Deep graph infomax.'' in \emph{ICLR}, 2019.

\bibitem{hassani2020contrastive}
K.~Hassani and A.~H. Khasahmadi, ``Contrastive multi-view representation
  learning on graphs,'' in \emph{ICML}, 2020, pp. 4116--4126.

\bibitem{schober2018correlation}
P.~Schober, C.~Boer, and L.~A. Schwarte, ``Correlation coefficients:
  appropriate use and interpretation,'' \emph{Anesthesia \& Analgesia}, vol.
  126, no.~5, pp. 1763--1768, 2018.

\bibitem{savva2019assessment}
A.~D. Savva, G.~D. Mitsis, and G.~K. Matsopoulos, ``Assessment of dynamic
  functional connectivity in resting-state fmri using the sliding window
  technique,'' \emph{Brain and behavior}, vol.~9, no.~4, p. 1255, 2019.

\bibitem{shakil2016evaluation}
S.~Shakil, C.-H. Lee, and S.~D. Keilholz, ``Evaluation of sliding window
  correlation performance for characterizing dynamic functional connectivity
  and brain states,'' \emph{Neuroimage}, vol. 133, pp. 111--128, 2016.

\bibitem{8765628}
T.-E. Kam, H.~Zhang, Z.~Jiao, and D.~Shen, ``Deep learning of static and
  dynamic brain functional networks for early mci detection,'' \emph{IEEE
  Transactions on Medical Imaging}, vol.~39, no.~2, pp. 478--487, 2020.

\bibitem{thakoor2021bootstrapped}
S.~Thakoor, C.~Tallec, M.~G. Azar, M.~Azabou, E.~L. Dyer, R.~Munos,
  P.~Veličković, and M.~Valko, ``Large-scale representation learning on
  graphs via bootstrapping,'' \emph{arXiv preprint arXiv:2102.06514}, 2021.

\bibitem{zhang2021canonical}
H.~Zhang, Q.~Wu, J.~Yan, D.~Wipf, and P.~S. Yu, ``From canonical correlation
  analysis to self-supervised graph neural networks,'' in \emph{NIPS}, vol.~34,
  2021.

\bibitem{kipf2016semi}
T.~N. Kipf and M.~Welling, ``Semi-supervised classification with graph
  convolutional networks,'' in \emph{ICLR}, 2017.

\bibitem{GAT}
P.~Veli{\v{c}}kovi{\'c}, G.~Cucurull, A.~Casanova, A.~Romero, P.~Lio, and
  Y.~Bengio, ``Graph attention networks,'' in \emph{ICLR}, 2018.

\bibitem{GINConv}
K.~Xu, W.~Hu, J.~Leskovec, and S.~Jegelka, ``How powerful are graph neural
  networks?'' \emph{arXiv preprint arXiv:1810.00826}, 2018.

\bibitem{wei2020theoretical}
C.~Wei, K.~Shen, Y.~Chen, and T.~Ma, ``Theoretical analysis of self-training
  with deep networks on unlabeled data,'' in \emph{ICLR}, 2021.

\bibitem{zhuang2020technical}
X.~Zhuang, Z.~Yang, and D.~Cordes, ``A technical review of canonical
  correlation analysis for neuroscience applications,'' \emph{Human Brain
  Mapping}, vol.~41, no.~13, pp. 3807--3833, 2020.

\bibitem{tishby2000information}
N.~Tishby, F.~C. Pereira, and W.~Bialek, ``The information bottleneck method,''
  \emph{arXiv preprint arXiv:0004057}, 2000.

\bibitem{ghojogh2021kkt}
B.~Ghojogh, A.~Ghodsi, F.~Karray, and M.~Crowley, ``Kkt conditions, first-order
  and second-order optimization, and distributed optimization: Tutorial and
  survey,'' \emph{arXiv preprint arXiv:2110.01858}, 2021.

\bibitem{lee2021predicting}
J.~D. Lee, Q.~Lei, N.~Saunshi, and J.~Zhuo, ``Predicting what you already know
  helps: Provable self-supervised learning,'' in \emph{NIPS}, vol.~34, 2021.

\bibitem{makur2015efficient}
A.~Makur, F.~Kozynski, S.-L. Huang, and L.~Zheng, ``An efficient algorithm for
  information decomposition and extraction,'' in \emph{Allerton}.\hskip 1em
  plus 0.5em minus 0.4em\relax IEEE, 2015, pp. 972--979.

\bibitem{song2021graph}
X.~Song, F.~Zhou, A.~F. Frangi, J.~Cao, X.~Xiao, Y.~Lei, T.~Wang, and B.~Lei,
  ``Graph convolution network with similarity awareness and adaptive
  calibration for disease-induced deterioration prediction,'' \emph{Medical
  Image Analysis}, vol.~69, p. 101947, 2021.

\bibitem{yan2010dparsf}
C.~Yan and Y.~Zang, ``Dparsf: a matlab toolbox for" pipeline" data analysis of
  resting-state fmri,'' \emph{Frontiers in Systems Neuroscience}, vol.~4,
  p.~13, 2010.

\bibitem{loshchilov2017decoupled}
I.~Loshchilov and F.~Hutter, ``Decoupled weight decay regularization,''
  \emph{arXiv preprint arXiv:1711.05101}, 2017.

\bibitem{yang2020revisiting}
C.~Yang, R.~Wang, S.~Yao, S.~Liu, and T.~Abdelzaher, ``Revisiting
  over-smoothing in deep gcns,'' \emph{arXiv preprint arXiv:2003.13663}, 2020.

\bibitem{saunshi2019theoretical}
N.~Saunshi, O.~Plevrakis, S.~Arora, M.~Khodak, and H.~Khandeparkar, ``A
  theoretical analysis of contrastive unsupervised representation learning,''
  in \emph{ICML}, 2019, pp. 5628--5637.

\bibitem{wang2019identifying}
M.~Wang, D.~Zhang, J.~Huang, P.-T. Yap, D.~Shen, and M.~Liu, ``Identifying
  autism spectrum disorder with multi-site fmri via low-rank domain
  adaptation,'' \emph{IEEE transactions on medical imaging}, vol.~39, no.~3,
  pp. 644--655, 2019.

\end{thebibliography}
\end{document}